\documentclass[11pt]{article}

\PassOptionsToPackage{numbers, compress}{natbib}
\usepackage[utf8]{inputenc}
\usepackage{amsmath}
\usepackage{amsfonts}
\usepackage{mathrsfs}
\usepackage{amssymb}
\usepackage{amsthm}

\usepackage{libertine}
\usepackage{wrapfig}
\usepackage{graphicx}
\usepackage{multirow}
\usepackage{epstopdf}
\usepackage{multicol}
\usepackage{subfigure}
\usepackage{booktabs}

\usepackage{algorithmic}
\usepackage{algorithm}
\usepackage[algo2e,linesnumbered]{algorithm2e}
\usepackage{stackengine}
\usepackage{graphicx}
\usepackage{multirow}
\usepackage{epstopdf}
\usepackage{multicol}
\usepackage{caption}
\usepackage{subfigure}
\usepackage{enumitem}
\usepackage{amsmath}
\usepackage{pifont}
\usepackage{natbib}
\usepackage{bbm}
\usepackage{makecell}
\usepackage{mathtools}
\usepackage[table]{xcolor}
\definecolor{mydarkred}{rgb}{0.6,0,0}
\definecolor{mydarkgreen}{rgb}{0,0.6,0}
\usepackage[colorlinks,
linkcolor=mydarkred,
citecolor=mydarkgreen]{hyperref}
\usepackage{url}
\usepackage{bm}
\usepackage{pifont}
\usepackage{stfloats}
\usepackage[T1]{fontenc}
\usepackage{arydshln}
\usepackage{colortbl}
\usepackage{balance}

\newtheorem{theorem}{Theorem}

\newcolumntype{L}[1]{>{\raggedright\let\newline\\\arraybackslash\hspace{0pt}}m{#1}}
\newcolumntype{Y}{>{\centering\arraybackslash}X}
\newcolumntype{s}{>{\hsize=.3\hsize}Y}
\newcolumntype{t}{>{\hsize=1.5\hsize}X}
\newcolumntype{u}{>{\hsize=0.8\hsize}Y}
\usepackage{makecell}
\usepackage{mathtools}
\usepackage[utf8]{inputenc}
\usepackage{stfloats}
\newcommand*{\defeq}{\mathrel{\vcenter{\baselineskip0.5ex \lineskiplimit0pt
                     \hbox{\scriptsize.}\hbox{\scriptsize.}}}%
                     =}

\title{Kernel Mean Estimation by\\ Marginalized Corrupted Distributions}

\author{
  Xiaobo Xia$^{1}\thanks{Equal contributions.}$,
  Shuo Shan$^{2}\footnotemark[1]$,
  Mingming Gong$^3$,\\
  Nannan Wang$^4$,
  Fei Gao$^5$,
  Haikun Wei$^{2}$,
  Tongliang Liu$^{1}$\\
  \small{$^1$The University of Sydney;}
  \small{$^2$Southeast University;}\\
  \small{$^3$The University of Melbourne;}
  \small{$^4$Xidian University;}\\
  \small{$^5$Hangzhou Dianzi University}
}
\date{}
\begin{document}

\maketitle

\begin{abstract}
Estimating the kernel mean in a reproducing kernel Hilbert space is a critical component in many kernel learning algorithms. Given a finite sample, the standard estimate of the target kernel mean is the empirical average. Previous works have shown that better estimators can be constructed by shrinkage methods. In this work, we propose to corrupt data examples with noise from known distributions and present a new kernel mean estimator, called the marginalized kernel mean estimator, which estimates kernel mean under the corrupted distribution. Theoretically, we show that the marginalized kernel mean estimator introduces implicit regularization in kernel mean estimation. Empirically, we show on a variety of datasets that the marginalized kernel mean estimator obtains much lower estimation error than the existing estimators.
\end{abstract}

\newpage

\section{Introduction}\label{sec:1}
The kernel mean, which is to the mean of a kernel function in a Reproducing Kernel Hilbert Space (RKHS) computed w.r.t. a distribution $\mathbb{P}$, has played a fundamental role in many kernel-based learning algorithms, ranging from traditional kernel component analysis \cite{Pearson1900On} to more recent Hilbert space embedding of distributions \cite{smola2007hilbert, muandet2017kernel}. A kernel mean of a distribution $\mathbb{P}$ over a measurable space $\mathcal{X}$ is defined by
\begin{equation}
\mu_{\mathbb{P}}\defeq\int_{\mathcal{X}} k(x, \cdot) \mathrm{d} \mathbb{P}(x)\in \mathcal{H}_k,
\label{Eq:kme_true}
\end{equation}
where $\mathcal{H}_k$ is a RKHS induced by the kernel function $k:\mathcal{X}\times\mathcal{X}\rightarrow \mathbb{R}$ that satisfies the reproducing property $\langle f, k(x, \cdot)\rangle_{\mathcal{H}_k}=f(x) ,~\forall f \in \mathcal{H}_k$. In practice, since the true distribution $\mathbb{P}$ is unknown, we can compute an empirical estimate of the kernel mean from an i.i.d. sample $S = \left\{x_{1},x_{2},...,x_{n}\right\}$ drawn from $\mathbb{P}$, by the following average
\begin{equation}
\widehat{\mu}_{\mathbb{P}}\defeq\frac{1}{n} \sum_{i=1}^{n} k\left(x_{i}, \cdot\right).
\label{Eq:kme_empirical} 
\end{equation}

Though kernel mean has been widely used in various kernel-based learning algorithms, such as kernel PCA \cite{scholkopf1997kernel}, kernel CCA \cite{bach2002kernel, andrew2013deep}, kernel FDA \cite{yang2005kpca}, kernel $k$-means\cite{dhillon2004kernel}, etc., its independent importance was noticed in the Hilbert space embedding of distributions \cite{muandet2017kernel}. Embedding a distribution as a kernel mean in RKHS has several benefits. First, a characteristic kernel $k$ ensures that $\left\|\mu_{\mathbb{P}}-\mu_{\mathbb{Q}}\right\|_{\mathcal{H}_k} = 0$ if and only if $\mathbb{P} = \mathbb{Q}$~\cite{sriperumbudur2008injective}. In other words, the mapping $\mathbb{P}\mapsto\mu_{\mathbb{P}}$ is injective and thus preserves all the information about the distribution. As a result, kernel mean embedding of distributions has been explored for two-sample tests \cite{gretton2012kernel} and independence tests \cite{gretton2005measuring,ramdas2015nonparametric}. Second, due to the reproducing property of a RKHS, the basic operations can be implemented by inner products in the space, i.e., $\mathbb{E}_\mathbb{P}[f(x)]=\langle f, \mu_{\mathbb{P}}\rangle_{\mathcal{H}_k}$. Hence, kernel mean embedding allows us to perform nonparametric probabilistic inference, such as kernel Bayes' rule \cite{fukumizu2013kernel}, kernel Belief propagation \cite{song2011kernel}, kernel adaptive MCMC \cite{gilks1998adaptive,andrieu2003ergodicity}, etc.

In all the applications of kernel mean estimation, an essential problem is to estimate it from finite data. Without any prior knowledge about $\mathbb{P}$, the empirical estimator (\ref{Eq:kme_empirical}) is
probably the best one we can obtain. However, motivated by the James-Stein shrinkage estimator \cite{stein1981estimation}, \cite{muandet2014kernel,muandet2014kernel_spectral,muandet2016kernel} showed that the average estimator can be improved by a shrinkage estimator which introduces additional bias to reduce the estimation variance and ultimately achieves bias-variance tradeoff. The shrinkage estimators were shown to have smaller mean squared error than the standard one (\ref{Eq:kme_empirical}) in various applications. These findings suggest that the empirical kernel mean estimator can be improved if proper prior knowledge or regularization were incorporated in the estimation and that there is abundant room for further improvement.

In this work, we propose a new kernel mean estimator by marginalized corrupted distributions, which we call marginalized kernel mean estimator (MKME). Our approach corrupts each data point with a fixed noise distribution and then computes the kernel mean on the corrupted distributions. Computing kernel mean on the corrupted distributions is equivalent to estimating it on infinitely many corrupted data points. By choosing appropriate noise distributions, our MKME can be efficiently computed with no additional computational cost. In addition, we show that our MKME can be approximated by an adaptive shrinkage estimator, which suggests that our method can be potentially better than existing shrinkage estimators. Finally, we demonstrate the effectiveness of our method on various applications that require estimation of kernel means.

This work was partially inspired by recent success on marginalized denoising auto-encoders \cite{chen2012marginalized,chen2014marginalized} and marginalized corrupted features for supervised learning \cite{maaten2013learning}. The trick of marginalizing out corruptions using expectation avoids explicitly generating the training samples, thereby maintaining computational efficiency \cite{chen2012marginalized}. To the best of our knowledge, our MKME is the first attempt to estimate kernel mean from marginalized corrupted data. Additionally, previous methods on marginalized corrupted learning are mainly concerned with learning a parametric nonlinear function, while our MKME is nonparametric, which does not involve learning model parameters. Finally, previous methods need to approximate the empirical risk in order to apply the marginalization trick, while our MKME performs marginalization analytically without any approximation.

The paper is organized as follows. Section \ref{Sec:2} reviews previous works on estimation and applications of kernel mean and learning by corrupting data. Section \ref{Sec:3} presents preliminaries on kernel mean estimation. Section \ref{Sec:4} formally presents the proposed marginalized kernel mean estimator with a series of analyses. In Section \ref{Sec:5}, we valid our approach on both synthetic and real data using different applications based on kernel mean embedding. The results show that marginalized kernel embedding outperforms other estimators on synthetic data and on most small-scale datasets. Finally, we conclude this paper in Section \ref{Sec:6}.  

\section{Related Work}\label{Sec:2}
In this section, we first review existing kernel mean estimators that improve upon the standard empirical estimator, especially the shrinkage estimators \cite{muandet2014kernel,muandet2014kernel_spectral,muandet2016kernel}. Then, we summarize several typical applications of kernel mean embedding of distributions. Last, we briefly discuss existing supervised and unsupervised learning methods that rely on corrupted data.

\subsection{Kernel Mean Estimation (KME)}
The standard kernel mean estimator  (\ref{Eq:kme_empirical}) has been adopted in almost all the existing kernel-based learning algorithms. It has been shown that this estimator is $\sqrt{n}$ a consistent estimator of $\mu_{\mathbb{P}}$ in $\mathcal{H}_k$ norm \cite{smola2007hilbert,gretton2012kernel,lopez2015towards}. \cite{tolstikhin2017minimax} showed that $O_{P}\left(n^{-1 / 2}\right)$ is minimax in $\|\cdot\|_{\mathcal{H}_{k}}$ norm over the class of discrete measures and
the class of measures that has an infinitely differentiable density, when $k$ is a continuous translation-invariant kernel on $\mathbb{R}^d$.  \cite{muandet2014kernel} first showed that the standard estimator can be outperformed by James-Stein-like shrinkage estimators \cite{james1992estimation}. By casting the kernel mean estimation problem as a regression problem, which fits into the empirical risk minimization (ERM) framework, \cite{muandet2014kernel} further proposed two shrinkage estimators called the simple kernel mean shrinkage estimator (S-KMSE) and the flexible kernel mean estimator (F-KMSE) through penalized ERM. The two estimators were shown to perform better than the standard one in various applications, including kernel mean estimation, density estimation, kernel PCA, and discriminative learning on distributions. By considering regularization from a filter function perspective, \cite{muandet2014kernel_spectral} proposed a wide class of shrinkage estimators that allows easy incorporation of prior knowledge by choosing appropriate filter functions. The filter-based shrinkage estimators were shown to be superior to the standard kernel mean estimator in terms of accuracy without sacrificing computational efficiency. These non-parametric shrinkage estimators are also $\sqrt{n}$ consistent in $\mathcal{H}_k$ norm for bounded continuous kernels on $\mathcal{X}$. \cite{flaxman2016bayesian} proposed a Bayesian framework
estimation of kernel mean embeddings, obtaining estimators with desirable shrinkage properties, allowing quantification of full posterior uncertainty, as well as automating kernel choice and hyperparameter selection.

\subsection{Applications of KME}
Kernel mean serves as a basic component to most kernel-based learning algorithms. For example, nonlinear component analysis algorithms, such as kernel PCA \cite{scholkopf1997kernel}and kernel CCA \cite{bach2002kernel, andrew2013deep}, rely on mean functions and covariance operators in RKHS. The kernel K-means algorithm \cite{dhillon2004kernel} performs clustering in feature space and represent the clusters by the mean functions. Kernel mean gained more attention due to the establishment of kernel mean embedding of distributions \cite{smola2007hilbert, muandet2017kernel}. 

As kernel mean embedding captures all the information about the distribution for {\it characteristic kernels}, i.e., the mapping $\mathbb{P}\mapsto\mu_{\mathbb{P}}$ is injective, \cite{gretton2012kernel} proposed Maximum Mean Discrepancy (MMD) which is the RKHS distance between the mean embeddings of two probability measures. MMD has been widely applied in various learning algorithms, such as nonparametric two-sample tests\cite{gretton2012kernel}, nonparametric independence tests\cite{ramdas2015nonparametric}, domain adaptation \cite{pan2010domain}, deep generative models \cite{li2017mmd}, etc. In addition, some elementary operations on distributions can be easily performed by using the kernel mean embeddings, which allows us to perform nonparametric probabilistic inference. For example, \cite{fukumizu2013kernel} proposed the kernel Bayes’ rule (KBR) which realizes Bayesian inference in completely nonparametric settings without any parametric assumptions. \cite{song2009hilbert,Song:2010:HSE:3104322.3104448} introduced the kernel mean embedding of nonlinear dynamic systems and hidden Markov models and developed nonparametric filtering algorithms for efficient inference. Finally, in the predictive learning on distributions problems, in which each input feature is a distribution, several works utilize kernel mean embedding as a representation of the input distributions and construct kernel-based learning algorithms \cite{muandet2012learning,Muandet:2013:OSM:3023638.3023684,szabo2015two}. We refer readers to \cite{muandet2017kernel} for a thorough review for kernel mean embedding and its applications.

\subsection{Learning by Marginalized Corrupted Data}
Marginalized approaches stem from a natural assumption that the augmenting training sample by adding randomly generated noise will decrease generalization error as well as increase robustness. Choosing a similar sample but different from the training one is known as data augmentation \cite{simard1998transformation}. According to Vicinal Risk Minimization principle, additional virtual or noisy examples can be drawn from a typical distribution to enlarge the training hypothesis spaces \cite{chapelle2001vicinal} and lead to improved generalization \cite{simard1998transformation, zhang2017mixup}. However, explicitly generating noisy examples can be computationally expensive. For example, generating 10 noisy examples for one training example will lead to a 10 growth on the size of training set. Fortunately, marginalized approaches provide an alternative solution in solving this trade-off. The trick is to marginalize out the expectation over the noise distribution for each training example. In other word, it transforms some of the original data points into corrupted ones without changing the total amount of data points. This idea is called marginalized corrupted features and introduced in \cite{maaten2013learning}. The similar idea is investigated in \cite{chen2014marginalized} previously for domain adaption and later for non-linear representations \cite{berlinet2011reproducing}. Additionally, recent deep learning methods employ the idea for adversarial learning \cite{zhang2017mixup}, semi-supervised learning \cite{miyato2018virtual}, and object localization \cite{yun2019cutmix}. Lots of works have shown the effectiveness of marginalized approaches, and thereby motivating us to employ it further into kernel methods.

\section{Preliminaries}\label{Sec:3}
In this section, we present preliminaries on kernel mean estimation. We first provide some notations related to RKHS and briefly describe some properties that enables kernel mean embedding of distributions (Section \ref{sec:3.1}). Then we give the formulation of the kernel mean estimator from a regression perspective (Section \ref{sec:3.2}).
\subsection{Notations \& Properties}\label{sec:3.1}
Let $\mathcal{X}$ be a separable topological space and $\mathcal{H}$ be a Hilbert space of functions $f:\mathcal{X}\rightarrow\mathbb{R}$ mapping $\mathcal{X}$ to $\mathbb{R}$. An evaluation functional over the Hilbert space of functions $\mathcal{H}$ is a linear functional $\mathcal{F}_x:\mathcal{H}\rightarrow \mathbb{R}$ that evaluates each function in the space at the point $x$, i.e., $\mathcal{F}_x[f]=f(x)$. $\mathcal{H}$ is an RKHS if the evaluation functionals are bounded and continuous, i.e., if there exists a $M$ such that $\mathcal{F}_x[f]\leq M\|f\|_{\mathcal{H}}$, where $\|\cdot\|_{\mathcal{H}}$ denotes the RKHS norm. Then, for each $x\in \mathcal{X}$ there exists a function $k_x\in\mathcal{H}$ with the reproducing property $\mathcal{F}_x[f]=f(x)=\langle f, k_x\rangle,~\forall f\in\mathcal{H}$. The reproducing kernel $k:\mathcal{X}\times\mathcal{X}\rightarrow\mathbb{R}$ of
$\mathcal{H}$ is $k(x,x')=k_x(x')=\langle k_x,k_{x'}\rangle$. Any symmetric and positive semi-definite kernel $k$ uniquely determines an RKHS \cite{aronszajn1950theory,muandet2016kernel}. One of the most widely used kernel functions is the Gaussian radial basis function (RBF) kernel defined as
\begin{flalign}\label{eq:rbf_kernel}
k(x,x')=\text{exp}\left(-\frac{\|x-x'\|^2}{2\theta^2}\right),\quad x,x'\in\mathcal{X},
\end{flalign}
where $\theta>0$ is the bandwidth parameter and $\|\cdot\|$ is the Euclidean norm.
\subsection{Kernel Mean Estimation as a Regression Problem}\label{sec:3.2}
The problem of kernel mean estimation is always regarded as a regression problem, which allows us to obtain the estimator efficiently \cite{muandet2014kernel,muandet2016kernel}. The kernel mean $\mu_{\mathbb{P}}$ and its empirical estimate $\hat{\mu}_{\mathbb{P}}$ can be obtained as a \textit{minimizer} of the following risk functionals respectively: 
\begin{equation}
\begin{aligned}
&R(g)\defeq\mathbb{E}_{x\sim\mathbb{P}}\|k(\cdot,x)-g\|_{\mathcal{H}}^2\ \ \text{and}\ \ \hat{R}(g)\defeq\frac{1}{n}\sum_{i=1}^n\|k(\cdot,x_i)-g\|_{\mathcal{H}}^2.
\end{aligned}
\end{equation}
We will call the estimator minimizing the empirical risk $\hat{R}(g)$ a \textit{kernel mean estimator} in this paper. Note that the risk $R(g)$ is different from the risk commonly considered, i.e., $\ell(\mu_\mathbb{P},g)=\|\mu_{\mathbb{P}}-g\|_{\mathcal{H}}^2$. Nevertheless, we have $\ell(\mu_\mathbb{P},g)=\mathbb{E}_{xx'}k(x,x')-2\mathbb{E}_xg(x)+\|g\|_\mathcal{H}^2$ and $R(g)=\mathbb{E}_xk(x,x)-2\mathbb{E}_xg(x)+\|g\|_\mathcal{H}^2$. The difference between the risk $\ell(\mu_\mathbb{P},g)$ and $R(g)$ only lies in $\mathbb{E}_{xx'}k(x,x')-\mathbb{E}_xk(x,x)$, but is not a function of $g$. The new form here introduce a more tractable cross-validation computation \cite{muandet2014kernel}, which will be presented later. Except for this, the resulting estimators are always evaluated w.r.t. the risk $\ell(\mu_\mathbb{P},g)$.

\section{Marginalized Kernel Mean Estimator}\label{Sec:4}
In this section, we formally introduce the proposed marginalized kernel mean estimator. Specifically, we first introduce how to obtain the proposed estimation with the marginalized approach (Section \ref{sec:4.1}). For different scenarios and goals, we then propose different variants of the marginalized kernel mean estimator (Section \ref{sec:4.2}).  Towards better understanding of the proposed marginalized estimator, we next discuss the linear approximation of the marginalized kernel mean estimator, which builds the connection to the shrinkage estimator (Section \ref{sec:4.3}). Then we show how to determine the covariance matrix based on the leave-one-out cross validation (LOOCV) approach (Section \ref{sec:4.4}). Finally, we present some kernel-based algorithms with marginalized estimators (Section \ref{sec:4.5}). 

\subsection{Marginalized Estimation of $\mu_{\mathbb{P}}$}\label{sec:4.1}
We can observe a sample $S=\{x_1,x_2,\ldots,x_n\}$ of size $n$ drawn independently and identically (i.i.d.) from a fixed distribution $\mathbb{P}$ defined over a separable topological space $\mathcal{X}$, with each example $x_i\in\mathbb{R}^d$, where $d$ denotes the dimension of each example. In the Vicinal Risk Minimization (VRM) principle \cite{chapelle2001vicinal}, the corresponding \textit{corrupted empirical distribution} can be specified by a \textit{corrupted distribution} with the density $v\left(\tilde{x}|x_i\right)$ for each example $x_{i}$, where $\tilde{x}$ denotes the virtual sample. Namely, the distribution $\mathbb{P}$ can be approximated by 
\begin{equation}
    \mathbb{P}_v(\tilde{x})=\frac{1}{n}\sum_{i=1}^n v(\tilde{x}|x_i).
\end{equation}
We regard the problem of kernel mean estimation as a regression problem as discussed, the proposed marginalized kernel mean estimator $\tilde{\mu}_\mathbb{P}$ can be obtained as a \textit{minimizer} of the \textit{empirical corrupted risk}: 
\begin{equation}\label{eq:corrupted_risk}
    \hat{R}_v(g)\defeq\frac{1}{n}\sum_{i=1}^n\mathbb{E}_{\tilde{x}\sim v(\tilde{x}|x_i)}\|k(\tilde{x},\cdot)-g\|_{\mathcal{H}}^2.
\end{equation}
In other words, 
\begin{equation}\label{eq:mkme}
    \tilde{\mu}_\mathbb{P}\defeq\arg\min_{g\in\mathcal{H}}\hat{R}_v(g)=\frac{1}{n}\sum_{i=1}^n\mathbb{E}_{\tilde{x}\sim v(\tilde{x}|x_i)}k(\tilde{x},\cdot).
\end{equation}
Comparing the marginalized estimation (\ref{eq:mkme}) and original empirical estimation (\ref{Eq:kme_empirical}) of the kernel mean $\mu_{\mathbb{P}}$, we can see that the marginalized estimation will reduce to the empirical estimation if the corrupted distribution $v(\tilde{x}|x_i)$ is a Dirac distribution, which means that the latter is a special case of the former. Also, the kernel feature mapping $k(x,\cdot)$ in (\ref{Eq:kme_empirical}) is replaced by a new kernel feature mapping in (\ref{eq:mkme}), i.e., $\tilde{k}(x,\cdot)=\mathbb{E}_{\tilde{x}\sim v(\tilde{x}| x)} k(\tilde{x},\cdot)$.

The marginalized estimator can be reviewed from a reverse way. Specifically, it first estimates the underlying distribution function in a non-parametric kernel density estimation approach, and then apply the kernel mean embedding to the estimated density. Besides, it provides a way to embed prior knowledge into the kernel mean embedding. 

\subsection{Single/Multi-variable Gaussian Marginalized Kernel Mean Estimator}\label{sec:4.2}
We introduce the corrupted distribution before presenting different variants of the marginalized kernel mean estimator. Specifically, in this paper, we exploit a multi-variate Gaussian distribution whose means are exactly the training examples and covariance matrix are unknown. Formally, the multi-variate Gaussian distribution, which is the corrupted distribution $v(\tilde{x}|x_i)$ with $x_i\in S$, is defined by 
\begin{equation}\label{eq:multi-variant}
    v\left(\tilde{x}|x_{i}\right)=\frac{1}{(2 \pi)^{d / 2}\left|\Sigma_{i}\right|^{1 / 2}} \exp \left(-\frac{1}{2}\left\|\tilde{x}-x_{i}\right\|_{\left(\Sigma_{i}\right)}^{2}\right),
\end{equation}
where $\Sigma_{i}$ is the corresponding covariance matrix and $\left\|\tilde{x}-x_{i}\right\|_{(\Sigma_{i})}^{2}=(\tilde{x}-x_i)^{\top}\Sigma_i^{-1}(\tilde{x}-x_i)$. Therefore, from the definition, only the covariance matrix of each Gaussian distribution is unknown and need to be determined.

The main problem is that there are too many parameters for the covariance matrix in comparison with the size of the training sample. The reason is, for each example, there are $d\times d$ variables in the corresponding covariance matrix. We need to reduce the size of parameters to a reasonable size so that the computation can be feasible and efficient. By posing restrictions on the covariance matrix, we will obtain a signiﬁcantly smaller size of parameters and thereby enabling feasible computation. 

Three restrictions on the covariance matrix can be used in this paper. \textit{i}) The covariance matrix of the corrupted distribution for each example is the same. With this restriction, the size of parameters for the covariance matrix will not enlarge exponentially as the size of training examples grows. \textit{ii}) The covariance matrix of the corrupted distribution is diagonal. This means that it limits the features of training examples to be independent with each other, sacrificing the potential correlation between features. \textit{iii}) The elements in the diagonal covariance matrix are the same. The above three restrictions may seem to be excessively strict at the first glance. However, adding the corrupted distribution is only an approach to reduce the generalization error and improve the robustness. We will still obtain performance improvement with these restrictions. Based on these restrictions, we further propose two corrupted distributions, which lead to two different marginalized kernel mean estimators. 

The first corrupted distribution is the \textit{single-variable Gaussian corrupted distribution} where we takes advantages of all the three restrictions. For this corrupted distribution, we call the corresponding marginalized kernel mean estimator as the \textit{single-variable Gaussian marginalized kernel mean estimator (abbreviated as MKME)}. With the mentioned three restrictions, the size of parameters degenerates to one dimension. Therefore, the covariance matrix $\Sigma_i$ becomes $\sigma^2I$ for all training examples. Accordingly, the distribution $v(\tilde{x}|x_i)$ for $x_i\in S$ is 
\begin{equation}
    v\left(\tilde{x}|x_{i}\right)=\frac{1}{(2 \pi)^{d / 2}\left|\sigma^2I\right|^{1 / 2}} \exp \left(-\frac{1}{2}\left\|\tilde{x}-x_{i}\right\|_{\left(\sigma^2I\right)}^{2}\right).
\end{equation}
The second corrupted distribution is the \textit{multi-variable Gaussian corrupted distribution} where we relax the restriction \textit{iii}) and only utilize the first two restrictions. We call the corresponding marginalized kernel mean estimator as the \textit{multi-variable Gaussian marginalized kernel mean estimator (abbreviated as MMKME)}. Let the covariance matrix be $D=\text{diag}(e_1,e_2,\ldots,e_d)$, the distribution $v(\tilde{x}|x_i)$ for $x_i\in S$ is 
\begin{equation}
    v\left(\tilde{x}|x_{i}\right)=\frac{1}{(2 \pi)^{d / 2}\left|D\right|^{1 / 2}} \exp \left(-\frac{1}{2}\left\|\tilde{x}-x_{i}\right\|_{\left(D\right)}^{2}\right).
\end{equation}
Note that in both situations, there exists some free variables in the corrupted distribution. In order to select appropriate values for the variables, we employ the \textit{leave-one-out cross validation (LOOCV)} approach, which will be introduced in Section \ref{sec:4.4}.
\subsection{Linear Approximation of Marginalized Estimators}\label{sec:4.3}
Although we can analytically compute the marginalized kernel mean embedding, it is also of interest to approximate the linear combinations as in shrinkage estimators \cite{muandet2014kernel,muandet2016kernel}, i.e., $\hat{\mu}_{\mathbb{P}}=\sum_{i=1}^n\beta_ik(x_i,\cdot)$, for some $\bm{\beta}\in\mathbb{R}^n$. On the one hand, we will gain more insights on the difference between the marginalized estimator and the previous estimators. On the other hand, the linear form can be more computationally efficient. 
\begin{theorem}
For the single-variable Gaussian marginalized kernel mean estimator (MKME), it has a linear form, i.e., $\tilde{\mu}_\mathbb{P}=\sum_{i=1}^n\beta_i k(x_i,\cdot)$. The weight vector $\bm{\beta}$ can be written as 
\begin{equation}
    \bm{\beta}=\frac{\theta^2+d\sigma^2}{2\theta^2}\bm{1}_{n}-\frac{d\sigma^2}{2\theta^4}\mathbf{K}^{-1}\mathbf{K'}\bm{1}_{n},
\end{equation}
where $\bm{1}_n=[1/n,1/n,\ldots,1/n]^\top$, $\mathbf{K}$ is an $n\times n$ Gram matrix such that $\mathbf{K}_{ij}=k(x_i,x_j)$, and $\mathbf{K}'$ denotes the kernel matrix generated by $k'(x,x')=\text{exp}(-\|x-x'\|/2\theta^2)\|x-x'\|^2$. 
\end{theorem}
\begin{proof}
We approximate (\ref{eq:corrupted_risk}) with its Taylor expansion up to the second order. For simplicity, let $\mathbb{E}_{\tilde{x}\sim v(\tilde{x}|x_i)}\|k(\tilde{x},\cdot)-g\|_{\mathcal{H}}^2$ be $\mathbb{E}_{\tilde{x}} \ell_\mathcal{H}(\tilde{x},g)$. Then, 
\begin{equation}
\begin{aligned}
    &\quad \ \mathbb{E}_{\tilde{x}\sim v(\tilde{x}|x_i)}\|k(\tilde{x},\cdot)-g\|_{\mathcal{H}}^2=\mathbb{E}_{\tilde{x}} \ell_\mathcal{H}(\tilde{x},g)\\
    &\approx\mathbb{E}_{\tilde{x}}[\ell_\mathcal{H}(x_i,g)+(\tilde{x}-x_i)^\top\nabla_{\tilde{x}}\ell_\mathcal{H}(x_i,g)-\frac{1}{2}(\tilde{x}-x_i)^\top\nabla_{\tilde{x}}^2\ell_\mathcal{H}(x_i,g)(\tilde{x}-x_i)]\\ 
    &=\ell_\mathcal{H}(x_i,g)-\frac{1}{2}\text{tr}\Big(\mathbb{E}_{\tilde{x}}[(\tilde{x}-x_i)(\tilde{x}-x_i)^\top]\Big)\nabla_{\tilde{x}}^2\ell_\mathcal{H}(x_i,g),
\end{aligned}
\end{equation}
where tr($\cdot$) denotes the trace of a matrix. Therefore, the marginalized estimator $\tilde{\mu}_\mathbb{P}$ in Eq.~ (\ref{eq:mkme}) can be be approximated as
\begin{equation}\label{eq:mkme_taylor}
\begin{aligned}
\tilde{\mu}_\mathbb{P}&=\arg \min_{g\in\mathcal{H}}\frac{1}{n}\sum_{i=1}^n\Big(\ell_\mathcal{H}(x_i,g)-\frac{1}{2}\text{tr}\Big(\mathbb{E}_{\tilde{x}}[(\tilde{x}-x_i)(\tilde{x}-x_i)^\top]\Big)\nabla_{\tilde{x}}^2\ell_\mathcal{H}(x_i,g)\Big).
\end{aligned}
\end{equation}
We also have 
\begin{equation}\label{eq:ell_x_g}
    \ell_\mathcal{H}(x_i,g)=\|k(x_i,\cdot)-g\|_{\mathcal{H}}^2=1-2\langle k(x_i,\cdot),g\rangle+\langle g,g\rangle.
\end{equation}
By the representer theorem \cite{scholkopf2001generalized}, Eq.~(\ref{eq:mkme_taylor}) has the linear form of $\sum_{i=1}^n\beta_i k(x_i,\cdot)$ for some $\bm{\beta}\in\mathbb{R}^n$. As we exploit the RBF kernel presented in Eq.~(\ref{eq:rbf_kernel}), combining Eq.~(\ref{eq:ell_x_g}), we have 
\begin{equation}
\begin{aligned}
\nabla_{\tilde{x}}^2\ell_\mathcal{H}(x_i,g)&=-2\frac{\partial^2\langle k(x_i,\cdot),g\rangle}{\partial\tilde{x}^2}=-2\sum_{i=1}^{n}\beta_i\frac{\partial^2k(\tilde{x},x_i)}{\partial\tilde{x}^2}\\
&=2\sum_{i=1}^{n}\beta_i\frac{\partial k(\tilde{x},x_i)\cdot\frac{1}{\theta^2}\|\tilde{x}-x_i\|}{\partial\tilde{x}}\\
&=\sum_{i=1}^n\beta_i k(\tilde{x},x_i)\Big(\frac{2}{\theta^2}-\frac{2}{\theta^4}\|\tilde{x}-x_i\|^2\Big). 
\end{aligned}
\end{equation}
As stated, for MKME, the corruption on each dimension is independently with a single variant denoted by $\sigma$. Therefore, we can further simplify the objective function and obtain
\begin{equation}
\begin{aligned}
&\quad \ \ell_\mathcal{H}(x_i,g)-\frac{1}{2}\text{tr}\Big(\mathbb{E}_{\tilde{x}}[(\tilde{x}-x_i)(\tilde{x}-x_i)^\top]\Big)\nabla_{\tilde{x}}^2\ell_\mathcal{H}(x_i,g)\\
&=\ell_\mathcal{H}(x_i,g)-\frac{d\sigma^2}{\theta^2}\sum_{i=1}^n\beta_i k(\tilde{x},x_i)+\frac{d\sigma^2}{\theta^4}\sum_{i=1}^n\beta_i k(\tilde{x},x_i)\|\tilde{x}-x_i\|^2\\
&=1-\frac{2\theta^2+d\sigma^2}{\theta^2}\bm{\beta}^\top\mathbf{K}_{\cdot i}+\frac{d\sigma^2}{\theta^4}\bm{\beta}^\top\mathbf{K}'_{\cdot i}+\bm{\beta}^\top\mathbf{K}\bm{\beta},
\end{aligned}
\end{equation}
where $\mathbf{K}$ represents the kernel matrix generated by the Gaussian RBF kernel, $\mathbf{K}'$ denotes the kernel matrix enerated by the kernel $k'(x,x')=\text{exp}(-\|x-x'|\|^2/2\theta^2)\|x-x'\|^2$, and $\mathbf{K}_{\cdot i}$ (resp. $\mathbf{K}'_{\cdot i}$) denotes the $i$-th column of the matrix $\mathbf{K}$ (resp. the $i$-th column of the matrix $\mathbf{K}'$). Therefore, recall Eq.~(\ref{eq:mkme_taylor}), we have 
\begin{equation}
\begin{aligned}
\tilde{\mu}_\mathbb{P}&=\arg \min_{g\in\mathcal{H}}\frac{1}{n}\sum_{i=1}^n\Big(\ell_\mathcal{H}(x_i,g)-\frac{1}{2}\text{tr}\Big(\mathbb{E}_{\tilde{x}}[(\tilde{x}-x_i)(\tilde{x}-x_i)^\top]\Big)\nabla_{\tilde{x}}^2\ell_\mathcal{H}(x_i,g)\Big)\\
&=\arg \min\bm{\beta}^\top\mathbf{K}\bm{\beta}-\bm{\beta}^\top\big(\frac{2\theta^2+d\sigma^2}{\theta^2}\mathbf{K}\bm{1}_{n}-\frac{d\sigma^2}{\theta^4}\mathbf{K}'\bm{1}_{n}\big),
\end{aligned}
\end{equation}
where $\bm{1}_n=[1/n,1/n,\ldots,1/n]^\top$. As the above optimization objective is a convex problem, we set its first derivative to zero and obtain the linear form of MKME
\begin{equation}
\tilde{\mu}_\mathbb{P}=\sum_{i=1}^n\beta_i k(x_i,\cdot),
\end{equation}
where 
\begin{equation}
\bm{\beta}=\frac{\theta^2+d\sigma^2}{2\theta^2}\bm{1}_{n}-\frac{d\sigma^2}{2\theta^4}\mathbf{K}^{-1}\mathbf{K}'\bm{1}_{n}.
\end{equation}
The proof is finished.
\end{proof}
Therefore, we obtain the linear approximation for MKME by Taylor expansion. It seems that there is a direct shrinkage on $\bm{\beta}$. In other words, the weights are controlled by $\sigma$, which is the covariance of the corrupted distribution. Accordingly, it should share the similar properties with shrinkage estimators \cite{muandet2014kernel,muandet2016kernel}, following regularization in kernel mean estimation, computational complexity, and convergence. As MKME and MMKME only differs in the covariance matrix restrictions, the linear approximation of MMKME can be obtained similarly. 

\begin{theorem}
For the multi-variable Gaussian marginalized kernel mean estimator (MMKME), it has a linear form, i.e., $\tilde{\mu}_\mathbb{P}=\sum_{i=1}^n\beta_i k(x_i,\cdot)$. Let  the  covariance  matrix of multi-variable Gaussian corrupted  distribution be $D=\text{diag}(e_1,e_2,\ldots,e_d)$. The weight vector $\bm{\beta}$ can be written as 
\begin{equation}
    \bm{\beta}=\frac{\theta^2+\sum_{j=1}^d e_j^2}{2\theta^2}\bm{1}_{n}-\frac{\sum_{j=1}^d e_j^2}{2\theta^4}\mathbf{K}^{-1}\mathbf{K'}\bm{1}_{n},
\end{equation}
where $\bm{1}_n=[1/n,1/n,\ldots,1/n]^\top$, $\mathbf{K}$ is an $n\times n$ Gram matrix such that $\mathbf{K}_{ij}=k(x_i,x_j)$, and $\mathbf{K}'$ denotes the kernel matrix generated by $k'(x,x')=\text{exp}(-\|x-x'\|/2\theta^2)\|x-x'\|^2$. 
\end{theorem}
\begin{proof}
For MMKME, as discussed, the corruption is determined by a covariance matrix $D=\text{diag}(e_1,e_2,\ldots,e_d)$. Similar with the proof of Theorem 1, we can obtain 
\begin{equation}
\begin{aligned}
&\quad \ \ell_\mathcal{H}(x_i,g)-\frac{1}{2}\text{tr}\Big(\mathbb{E}_{\tilde{x}}[(\tilde{x}-x_i)(\tilde{x}-x_i)^\top]\Big)\nabla_{\tilde{x}}^2\ell_\mathcal{H}(x_i,g)\\
&=\ell_\mathcal{H}(x_i,g)-\frac{\sum_{j=1}^d e_j^2}{\theta^2}\sum_{i=1}^n\beta_i k(\tilde{x},x_i)+\frac{\sum_{j=1}^d e_j^2}{\theta^4}\sum_{i=1}^n\beta_i k(\tilde{x},x_i)\|\tilde{x}-x_i\|^2\\
&=1-\frac{2\theta^2+\sum_{j=1}^d e_j^2}{\theta^2}\bm{\beta}^\top\mathbf{K}_{\cdot i}+\frac{\sum_{j=1}^d e_j^2}{\theta^4}\bm{\beta}^\top\mathbf{K}'_{\cdot i}+\bm{\beta}^\top\mathbf{K}\bm{\beta}.
\end{aligned}
\end{equation}
The definitions of $\mathbf{K}$, $\mathbf{K}'$, $\mathbf{K}_{\cdot i}$, and $\mathbf{K}'_{\cdot i}$ are the same as those in the proof of Theorem 1. Therefore, we have 
\begin{equation}
\begin{aligned}
\tilde{\mu}_\mathbb{P}&=\arg \min\bm{\beta}^\top\mathbf{K}\bm{\beta}-\bm{\beta}^\top\big(\frac{2\theta^2+\sum_{j=1}^d e_j^2}{\theta^2}\mathbf{K}\bm{1}_{n}-\frac{\sum_{j=1}^d e_j^2}{\theta^4}\mathbf{K}'\bm{1}_{n}\big),
\end{aligned}
\end{equation}
where $\bm{1}_n=[1/n,1/n,\ldots,1/n]^\top$. Due to the convexity of the above optimization problem, we can obtain the linear form of MMKME like the process of MKME, but with 
\begin{equation}
    \bm{\beta}=\frac{\theta^2+\sum_{j=1}^d e_j^2}{2\theta^2}\bm{1}_{n}-\frac{\sum_{j=1}^d e_j^2}{2\theta^4}\mathbf{K}^{-1}\mathbf{K'}\bm{1}_{n}.
\end{equation}
The proof is finished.
\end{proof}

\subsection{Determine the Covariance Matrix}\label{sec:4.4}
We use an automatic leave-one-out cross validation (LOOCV) approach to determine the covariance matrix for marginalized estimators. For a given covariance matrix $\Sigma_i$, denote by $\tilde{\mu}_{\Sigma_i}^{(-i)}$ as the kernel mean estimated from $\{x_j\}_{j=1}^n\backslash\{x_i\}$. We will measure the quality of $\tilde{\mu}_{\Sigma_i}^{(-i)}$ by how well it approximates $k(x_i,\cdot)$ with the overall quality being quantified by the cross-validation score,
\begin{equation}\label{eq:loocv}
    LOOCV(\Sigma_i)=\frac{1}{n}\sum_{i=1}^n\|k(x_i,\cdot)-\tilde{\mu}_{\Sigma_i}^{(-i)}\|_{\mathcal{H}}^2.
\end{equation}
The LOOCV formulation in Eq.~(\ref{eq:loocv}) differs from the one used in regression. To be specific, in regression, we measure the deviation of the prediction made by the function on the omitted observation. In LOOCV, we measure the deviation between the feature map of the omitted observation and the function itself. 

For the proposed marginalized estimators, in the above optimization problem, we need to determine $\Sigma_i$, $i=1,2,\ldots,n$. The computation over Hilbert norm for the empirical kernel embedding can be calculated, and \cite{song2009hilbert} provides the formula for the marginalized Hilbert norm. Therefore, we have 
\begin{equation}
\begin{aligned}
\Sigma^{\star}_i&=\arg\min_{\Sigma_i} LOOCV(\Sigma_i)\\
&=\arg\min_{\Sigma_i}\frac{1}{n}\sum_{i=1}^n\Big(k(x_i,x_i)-\frac{2}{n-1}\sum_{j\neq i}\mathbb{E}_{\tilde{x}\sim v(\tilde{x}|x_j)}k(\tilde{x},x_i)\\
&+\frac{1}{(n-1)^2}\sum_{j\neq i,k\neq i}\mathbb{E}_{\tilde{x}\sim v(\tilde{x}|x_j),\tilde{x}'\sim v(\tilde{x}'|x_k)}k(\tilde{x},\tilde{x}')\Big)\\
&=\arg\min_{\Sigma_i}\frac{1}{n}\sum_{i=1}^n\Big(k(x_i,x_i)-\frac{2}{n-1}\sum_{j\neq i}\mathbf{L}_{ji}+\frac{1}{(n-1)^2}\sum_{j\neq i,k\neq i}\mathbf{Q}_{jk}\Big),
\end{aligned}
\end{equation}
where $\mathbf{L}$ denotes the expected functional over one corrupted distribution,
\begin{equation}
\begin{aligned}
\mathbf{L}_{ij}&=\mathbb{E}_{\tilde{x}\sim v(\tilde{x}|x_j)}k(\tilde{x},x_i)=\frac{\theta^d}{|\Sigma_j+\theta^2I|^{1/2}}\text{exp}\left(-\frac{1}{2}\|x_i-x_j\|^2_{(\Sigma_j+\theta^2 I)}\right),
\end{aligned}
\end{equation}
and $\mathbf{Q}$ denotes the expected functional over two corrupted distributions, 
\begin{equation}\label{eq:Q}
\begin{aligned}
\mathbf{Q}_{ij}&=\mathbb{E}_{\tilde{x}\sim v(\tilde{x}|x_i),\tilde{x}'\sim v(\tilde{x}'|x_j)}k(\tilde{x},\tilde{x}')=\frac{\theta^d}{|\Sigma_i+\Sigma_j+\theta^2I|^{1/2}}\text{exp}\left(-\frac{1}{2}\|x_i-x_j\|^2_{(\Sigma_i+\Sigma_j+\theta^2 I)}\right). 
\end{aligned}
\end{equation}
The term $\mathbf{L}$ and $\mathbf{Q}$ correspond to different kinds of marginalized kernels. For MKME, the optimization problem only has one variable $\sigma$ such that $\Sigma_i=\sigma^2 I$, $i=1,2,\ldots,n$, and thus can be computed eﬃciently. For MMKME, we have a $d$-dimentional variable where $\Sigma_i=D=\text{diag}(e_1,e_2,\ldots,e_d)$, $i=1,2,\ldots,n$. The optimization function for MMKME is similar with MKME, and $\mathbf{L}$ and $\mathbf{Q}$ become slightly different.

The formulations of $\mathbf{L}$ and $\mathbf{Q}$ are obtained with expansions of kernels \cite{song2008tailoring} and the following equations:
\begin{equation}\label{eq:zonklar}
\begin{aligned}
    \mathbb{E}_{\tilde{x}\sim v(\tilde{x}|x_i)}k(\tilde{x},\cdot)=\int_{\mathcal{X}}k(\tilde{x},\cdot)v(\tilde{x}|x_i)\mathrm{d}\tilde{x}.
\end{aligned}
\end{equation}
The only difference between two corrupted distributions is the size of unknown parameters. Following \cite{muandet2014kernel,muandet2016kernel}, we can use the \textit{fminsearch} and \textit{fminbnd} routines of the MATLAB optimization toolbox to search the minimum of the optimization problems.
\subsection{Kernel-based Algorithms with Marginalized Estimators}\label{sec:4.5}
In this subsection, with the proposed marginalized estimators, we present the formulas of some kernel-based learning algorithms, e.g., maximum mean discrepancy (MMD) \cite{borgwardt2006integrating} and Hilbert-Schmidt independence criterion (HSIC) \cite{gretton2005measuring}. 
\subsubsection{MMD}
We first discuss MMD, which is an effective nonparametric metric for measuring the distributions based on two sets of data. Specifically, given i.i.d samples $S_1=\{x^1_1,x^1_2,\ldots,x^1_m\}$ and $S_2=\{x^2_1,x^2_2,\ldots,x^2_n\}$ from two probability distributions $\mathbb{P}$ and $\mathbb{Q}$ respectively, we can write an unbiased estimate of the MMD entirely in terms of the kernel $k$ as did in \cite{borgwardt2006integrating}:
\begin{equation}
\begin{aligned}
\widehat{\text{MMD}}^2[\mathcal{H},S_1,S_2]&=\frac{1}{m(m-1)}\sum_{i=1}^m\sum_{j\neq i}^m k(x_i^1,x_j^1)\\
&+\frac{1}{n(n-1)}\sum_{i=1}^n\sum_{j\neq i}^n k(x_i^2,x_j^2)\\
&-\frac{2}{mn}\sum_{i=1}^m\sum_{j=1}^n k(x_i^1,x_j^2).
\end{aligned}
\end{equation}
With the proposed marginalized estimators, the marginalized MMD is provided as 
\begin{equation}
\begin{aligned}
\widetilde{\text{MMD}}^2[\mathcal{H},\tilde{S}_1,\tilde{S}_2]&=\frac{1}{m(m-1)}\sum_{i=1}^m\sum_{j\neq i}^m \tilde{k}(x_i^1,x_j^1)\\
&+\frac{1}{n(n-1)}\sum_{i=1}^n\sum_{j\neq i}^n \tilde{k}(x_i^2,x_j^2)\\
&-\frac{2}{mn}\sum_{i=1}^m\sum_{j=1}^n \tilde{k}(x_i^1,x_j^2),
\end{aligned}
\end{equation}
where $\tilde{S}_1$ and $\tilde{S}_2$ denotes the corrupted version of $S_1$ and $S_2$ respectively, and $\tilde{k}$ denotes the the marginalized kernel which corresponds to Eq.~(\ref{eq:Q}). We thus have
\begin{equation}\label{eq:m_kernel}
\begin{aligned}
\tilde{k}(x_i,x_j)&=\mathbb{E}_{\tilde{x_i}\sim v(\tilde{x}_i|x_i),\tilde{x_j}\sim v(\tilde{x}_j|x_j)} k(x_i,x_j)\\
&=\frac{\theta^d}{|\Sigma_i+\Sigma_j+\theta^2 I|^{1/2}}\text{exp}(-\frac{1}{2}\|x_i-x_j\|^2_{(\Sigma_i+\Sigma_j+\theta^2 I)}).
\end{aligned}
\end{equation}
For different marginalized estimators, i.e., MKME and MMKME, we can set different covariances for them. 
\subsubsection{HSIC}
We discuss HSIC by now, which is one of the most successful non-parametric dependency measures. An empirical unbiased estimate of the HSIC statistic from an i.i.d. sample $\{(x_1,y_1),(x_2,y_2),\ldots,(x_n,y_n)\}$ on the measurable space $\mathcal{X}\times\mathcal{Y}$ can be obtained as follows:

\begin{equation}
\begin{aligned}
\widehat{\text{HSIC}}[\mathcal{X},\mathcal{Y},k,z]&=\frac{1}{n^2}\sum_{i,j=1}^n k(x_i,x_j)z(y_i,y_j)\\
&-\frac{2}{n^3}\sum_{i,j,p}^n k(x_i,x_j)z(y_i,y_p)\\
&+\frac{1}{n^4}\sum_{i,j=1}^n k(x_i,x_j) \sum_{p,q=1}^n z(y_p,y_q),
\end{aligned}
\end{equation}
where $z:\mathcal{Y}\times\mathcal{Y}\rightarrow \mathbb{R}$ is a positive definite kernel, and the product kernel $k(\cdot,\cdot)\times z(\cdot,\cdot)$ is characteristic on $\mathcal{X}\times\mathcal{Y}$. Note that we can also obtain the empirical unbiased estimate of the HSIC statistic as follows:
\begin{equation}
\widehat{\text{HSIC}}[\mathcal{X},\mathcal{Y},k,z]=\frac{1}{n^2}\text{tr}(\bar{\mathbf{K}}\bar{\mathbf{Z}}),
\end{equation}
where $\bar{\mathbf{K}}=\mathbf{HKH}$, $\bar{\mathbf{Z}}=\mathbf{HZH}$, and $\mathbf{H}=I-(1/n)\bm{1}\bm{1}^\top$. This matrix form can be computed in $O(n^2)$ times and converges to the population HSIC at a rate of $1/\sqrt{n}$ \cite{song2007supervised,muandet2016kernel}.

With the proposed marginalized estimators, the marginalized HSIC is provided as:
\begin{equation}
\widetilde{\text{HSIC}}[\mathcal{X},\mathcal{Y},k,z]=\frac{1}{n^2}\text{tr}(\breve{\mathbf{K}}\breve{\mathbf{Z}}),
\end{equation}
where $\breve{\mathbf{K}}=\mathbf{H}\tilde{\mathbf{K}}\mathbf{H}$ and $\breve{\mathbf{Z}}=\mathbf{H}\tilde{\mathbf{Z}}\mathbf{H}$. Here, $\tilde{\mathbf{K}}$ denotes the kernel matrix generated by the marginalized kernel $\tilde{k}(x_i,x_j)$ defined in Eq.~(\ref{eq:m_kernel}), and $\tilde{\mathbf{Z}}$ denotes the kernel matrix generated by the marginalized kernel $\tilde{z}(y_i,y_j)$.

\section{Experiments}\label{Sec:5}
In this section, we evaluate the performance of proposed marginalized kernel mean estimators on synthetic and real-world datasets. We verify its effectiveness on three representative applications, i.e., kernel two sample test, kernel density estimation, and Hilbert-Schmit independence test. We consider the following estimators as baselines: \textit{i}) empirical/standard kernel mean estimator (KME); \textit{ii}) simple kernel mean shrinkage estimator (S-KMSE); \textit{iii}) flexible kernel mean shrinkage estimator (F-KMSE). We exploit Gaussian RBF kernel which has been introduced in Section \ref{sec:3.1}. Unless otherwise stated, for all methods, we follow \cite{muandet2014kernel,muandet2016kernel} and set the bandwidth parameter of the Gaussian kernel as  $\theta^2=\text{median}\{\|x_i-x_j\|^2:i,j=1,2,\ldots,n\}$, i.e., the median heuristic. 
\subsection{Synthetic data} 
\subsubsection{Gaussian Distribution}
Given the true data-generating distribution $\mathbb{P}$ and the i.i.d. sample $S=\{x_1,x_2,\ldots,x_n\}$ from $\mathbb{P}$, we evaluate different estimators using the following loss function: 
\begin{equation*}
    \ell(\bm{\beta},S,\mathbb{P})\defeq\left\|\sum_{i=1}^n\beta_i k(x_i,\cdot)-\mathbb{E}_{x\sim\mathbb{P}} k(x,\cdot) \right\|_{\mathcal{H}}^2,
\end{equation*}
where $\bm{\beta}$ is the weight vector associated with different estimators. Following \cite{muandet2016kernel}, we can estimate the risk of the estimator by averaging over $m$ independent copies of $S$, i.e., $\hat{R}=\frac{1}{m}\sum_{j=1}^{m} \ell(\bm{\beta},S,\mathbb{P})$. 

To simulate a realistic situation, following \cite{muandet2014kernel}, we construct synthetic datasets by generating data from from a $d$-dimensional mixture of Gaussians:
\begin{equation}
\begin{aligned} x & \sim \sum_{i=1}^{4} \pi_{i} \mathcal{N}\left(\boldsymbol{\theta}_{i}, \Sigma_{i}\right)+\varepsilon, & \theta_{i j} & \sim \mathcal{U}(-10,10), \\ \Sigma_{i} & \sim \mathcal{W}\left(2 \times \mathbf{I}_{d}, 7\right), & & \varepsilon \sim \mathcal{N}\left(0,0.2 \times \mathbf{I}_{d}\right), \end{aligned}
\end{equation}
where $\mathcal{U}(a,b)$ and $\mathcal{W}(\Sigma_0,df)$ represent the uniform and Wishart distribution respectively. We set $\bm{\pi}$ as $[0.05,0.3,0.4,0.25]$ as did in \cite{muandet2014kernel}.

Note that the optimization minimizer will influence the quality of the kernel mean estimators. For our marginalized kernel mean estimators, the choice of the covariance as an important influence on the loss $\hat{R}$. Therefore, we first employ MKME which only has a parameter $\sigma$ to be determined, to show the importance of the covariance. As shown in Figure ~\ref{fig:choice_covariance}, we can see the significance of the choice of $\sigma$. Obviously, when $\sigma$ is chosen properly, the mean of the loss will decrease. As in complex real-world applications, we always cannot directly determine the covariance value. Thus, we still need to use optimization algorithms to find a suitable covariance value. 

\begin{figure}[!t]
    \centering
    \includegraphics[width=0.6\textwidth]{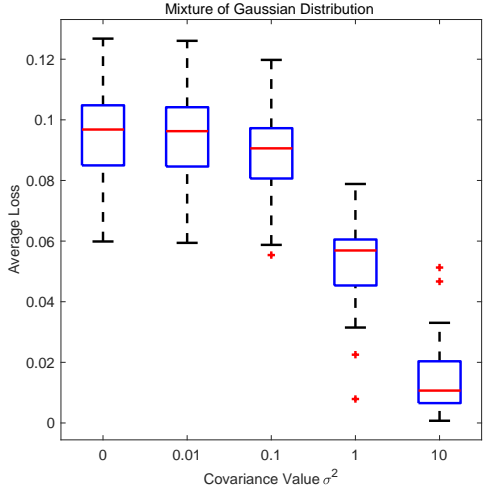}
    \caption{The illustration of MKME with different covariance variables. When the covariance value $\sigma^2=0$, it reduces to KME. We repeat the experiments over 30 different distributions with  $n$=10 and $d$=20.}
    \label{fig:choice_covariance}
\end{figure}
Figure~\ref{fig:mog_vary} depicts the average loss as we vary the sample size and dimension of the data. In this circumstance, both the shrinkage parameter and the covariance value are chosen by the leave-one-out cross-validation score. For the proposed MKME, we can see that it outperforms the baselines in most cases. For the proposed MMKME, it outperforms the baselines all the time. The experimental results verify the effectivess of the proposed marginalized estimators with mixture of Gaussians.\\
\begin{figure*}[!h]
\centering
\subfigure[]{
\begin{minipage}[t]{0.46\linewidth}
\centering
\includegraphics[width=6cm,height=5cm]{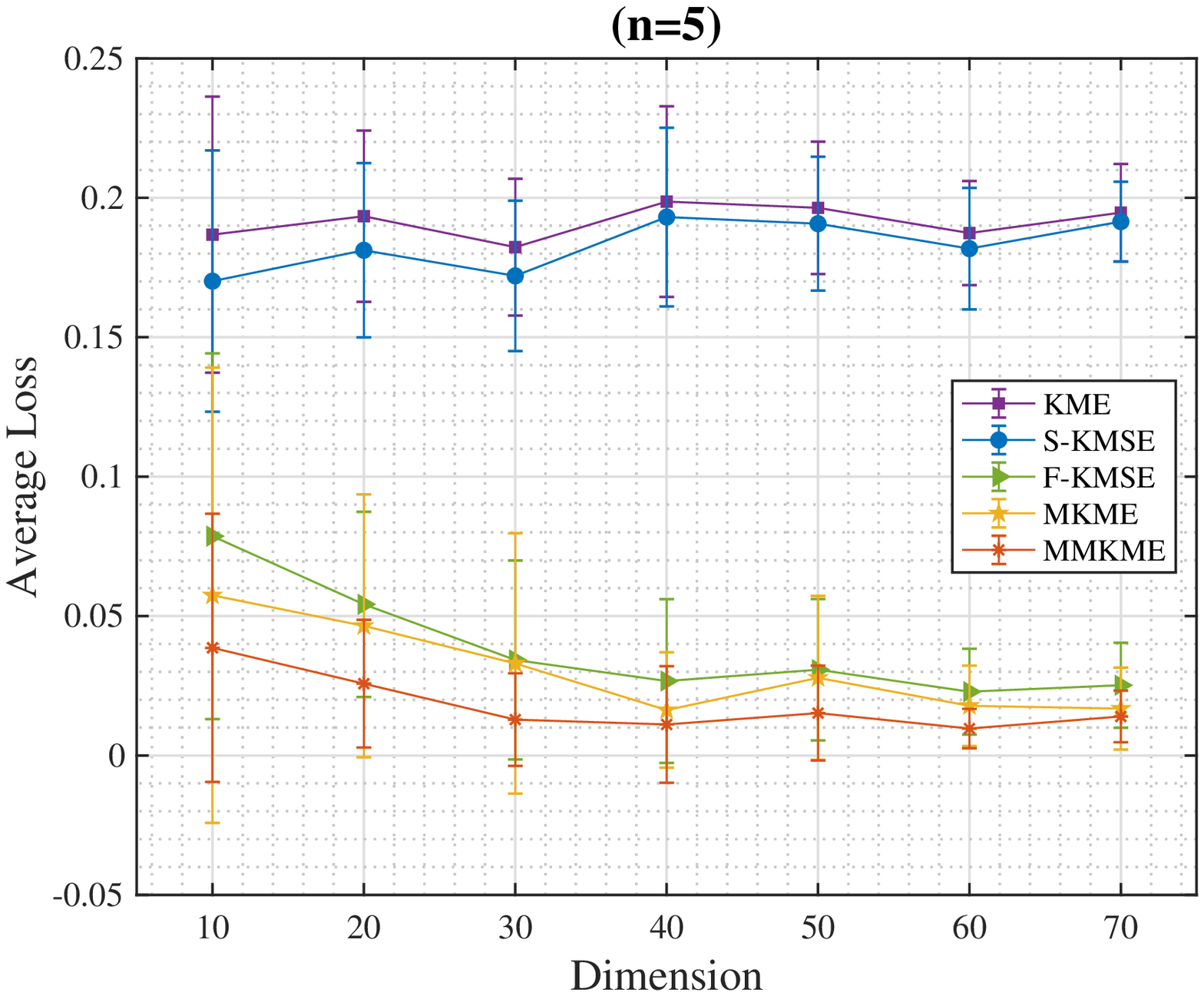}
\end{minipage}%
}%
\subfigure[]{
\begin{minipage}[t]{0.46\linewidth}
\centering
\includegraphics[width=6cm,height=5cm]{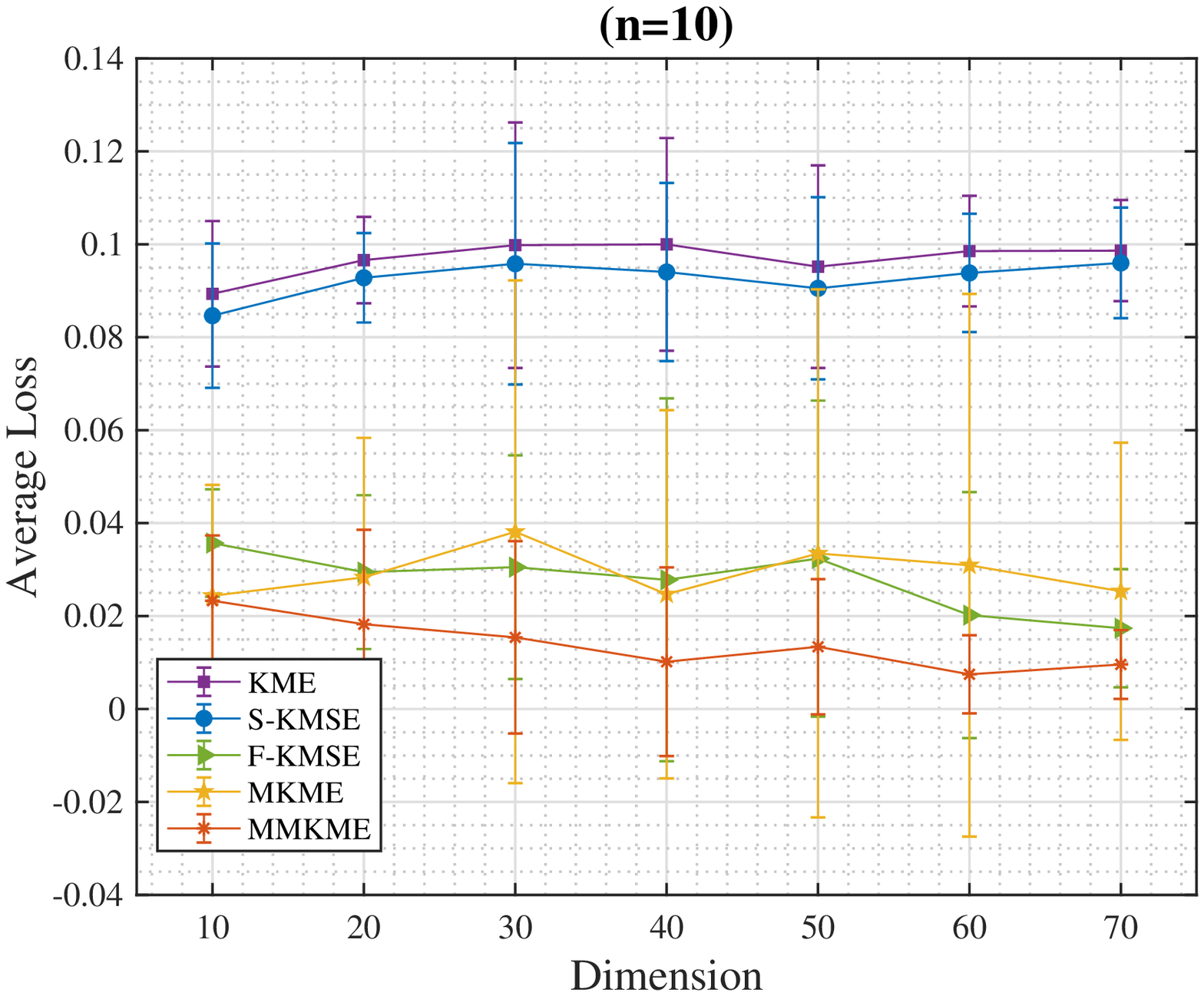}
\end{minipage}%
}%

\subfigure[]{
\begin{minipage}[t]{0.46\linewidth}
\centering
\includegraphics[width=6.3cm,height=5cm]{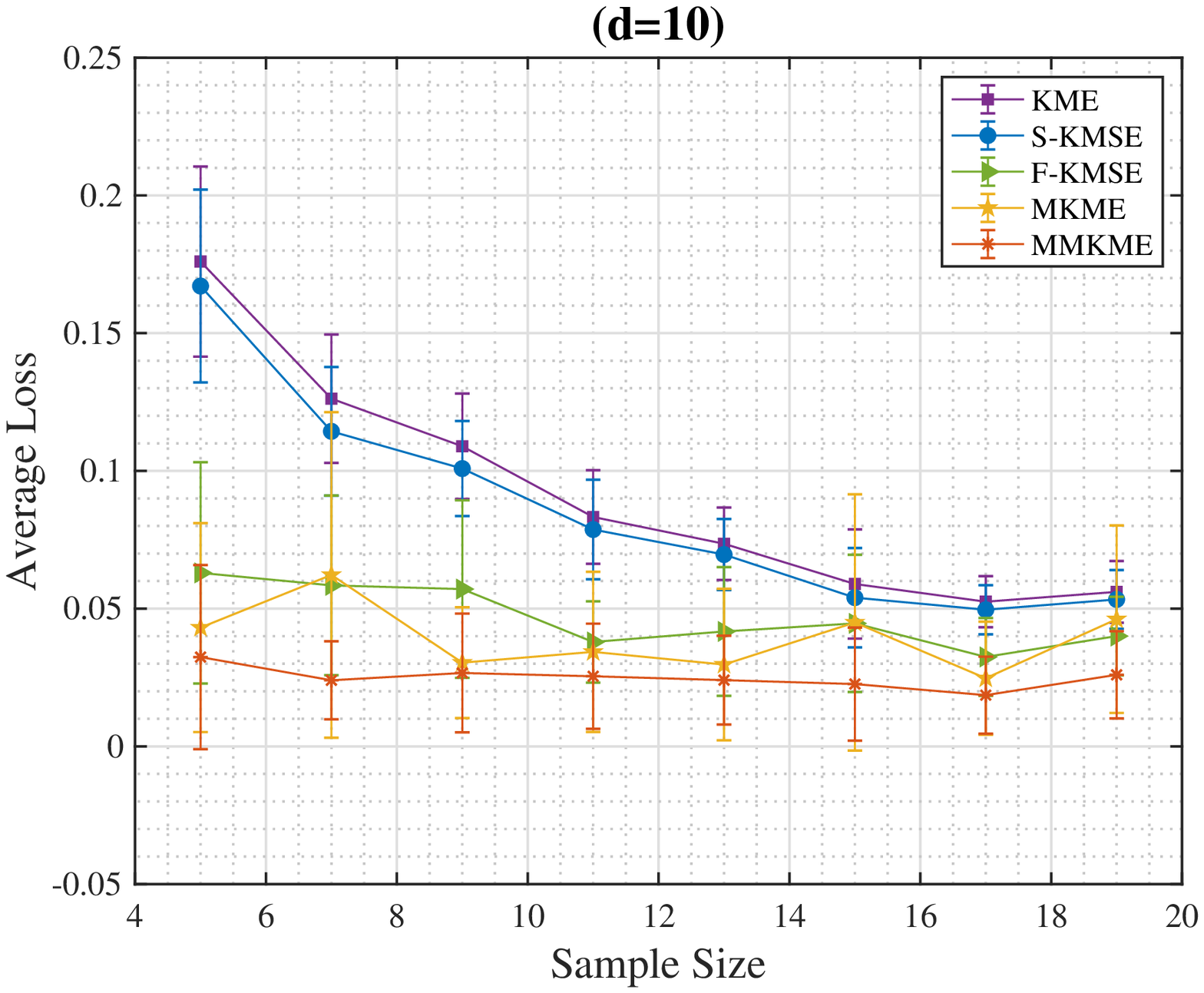}
\end{minipage}
}%
\subfigure[]{
\begin{minipage}[t]{0.46\linewidth}
\centering
\includegraphics[width=6.3cm,height=5cm]{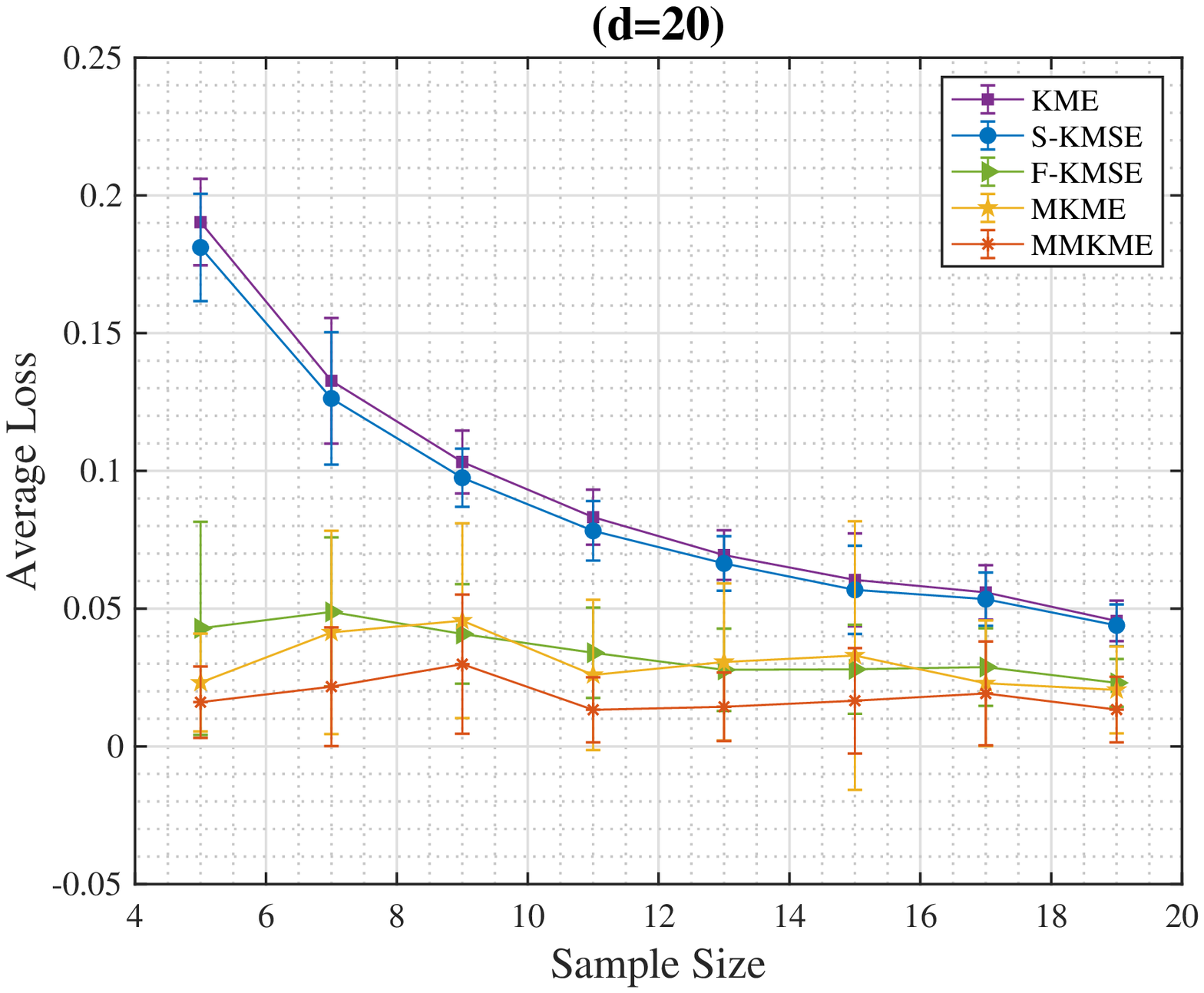}
\end{minipage}
}%
\centering
\caption{The average loss over 30 different distributions of KME, S-KMSE, F-KMSE, MKME, and MMKME with varying dimension ($d$) and sample size ($n$). Specifically, subfigure (a) and (b) investigate the average loss with varying $d$. Subfigure (c) and (d) investigate the average loss with varying $n$.}
\label{fig:mog_vary}
\end{figure*}
\subsubsection{T-distribution}
We generate a random sample from a T-distribution with a randomly generated covariance matrix. Under this setting, the loss between estimators and true kernel mean is hard to derive an analytical formula. We adopt the density estimation approach stated in the next section. The difference is that we are aware of the underlying distribution, thereby we can generate enough test data points for evaluation. We report the performance of different method by exploiting the  average negative log-likelihood (NLL). Smaller average negative log-likelihood corresponds better performance. We consider two cases: \textit{i}) varying the dimension $d$ with a fixed sample size $n$; \textit{ii}) varying the sample size $n$ with a fixed dimension $d$. The experimental results of the cases \textit{i} and \textit{ii} are presented in Table \ref{tab:nll_fix_n} and \ref{tab:nll_fix_d} respectively. As we can see, the proposed marginalized kernel mean estimators, i.e., MKME and MMKME, almost always outperform the baselines, which verifies the effectiveness of our method.

\begin{table*}[!h]
  \centering
  \normalsize
  \begin{tabular}{|c|ccccc|}
    \hline
    Dimension ($d$) &KME&S-KMSE&F-KMSE&MKME&MMKME\\
    \hline
    5&7.9538&7.8613&8.2966&7.8204&\textbf{7.6027}\\
    6&9.1763&9.2631&9.3052&9.0218&\textbf{9.0133}\\
    7&11.3587&11.2612&11.3377&\textbf{11.0209}&11.1506\\
    8&14.0386&13.2274&13.2631&13.2302&\textbf{13.0608}\\
    9&14.1509&14.0739&14.3102&14.0137&\textbf{14.0102}\\
    10&16.0481&15.8311&15.0027&15.0543&\textbf{14.9256}\\
    \hline
  \end{tabular}
  \caption{Average negative log-likelihood achieved by different methods over 30 different distributions. The sample size $n$ is fixed to 100. The boldface represents the best performance.}
  \label{tab:nll_fix_n}
\end{table*}

\begin{table*}[!h]
  \centering
  \normalsize
  \begin{tabular}{|c|ccccc|}
    \hline
    Sample Size ($n$)&KME&S-KMSE&F-KMSE&MKME&MMKME\\
    \hline
    15&36.1393&19.8977&\textbf{16.2368}&17.2840&17.5206\\
    30&19.3506&17.2619&16.0031&15.4565&\textbf{15.4375}\\
    60&16.8827&15.1718&14.8837&14.4377&\textbf{14.3526}\\
    90&14.9235&14.9216&14.6258&\textbf{14.4310}&14.5587\\
    120&14.4762&14.1780&14.2349&14.0212&\textbf{14.0138}\\
    150&14.5205&14.0224&14.0583&\textbf{13.8726}&13.9737\\
    \hline
  \end{tabular}
  \caption{Average negative log-likelihood achieved by different methods over 30 different distributions. The dimension $d$ is fixed to 10. The boldface represents the best performance.}
  \label{tab:nll_fix_d}
\end{table*}
\subsection{Real-World Applications}
To further evaluate the proposed estimators, we consider several benchmark applications, namely, kernel two sample test \cite{gretton2012kernel}, kernel density estimation \cite{song2008tailoring, botev2010kernel}, and Hilbert-Schmit independence test \cite{gretton2007kernel}. For some of these tasks, we employ datasets from the UCI repositories. 

\begin{figure*}[!h]
\centering
\subfigure[]{
\begin{minipage}[t]{0.46\linewidth}
\centering
\includegraphics[width=6cm,height=5cm]{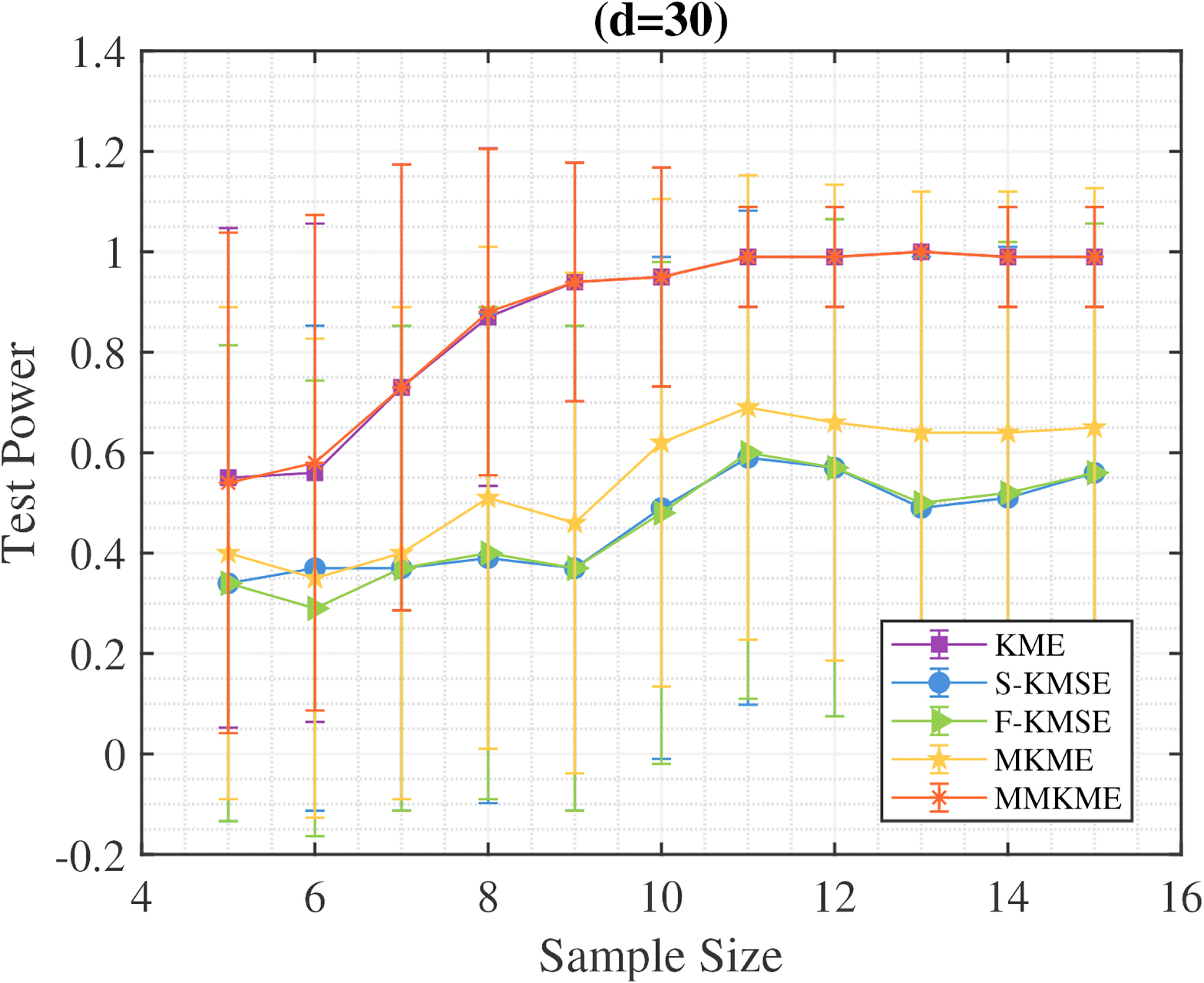}
\end{minipage}%
}%
\subfigure[]{
\begin{minipage}[t]{0.46\linewidth}
\centering
\includegraphics[width=6cm,height=5cm]{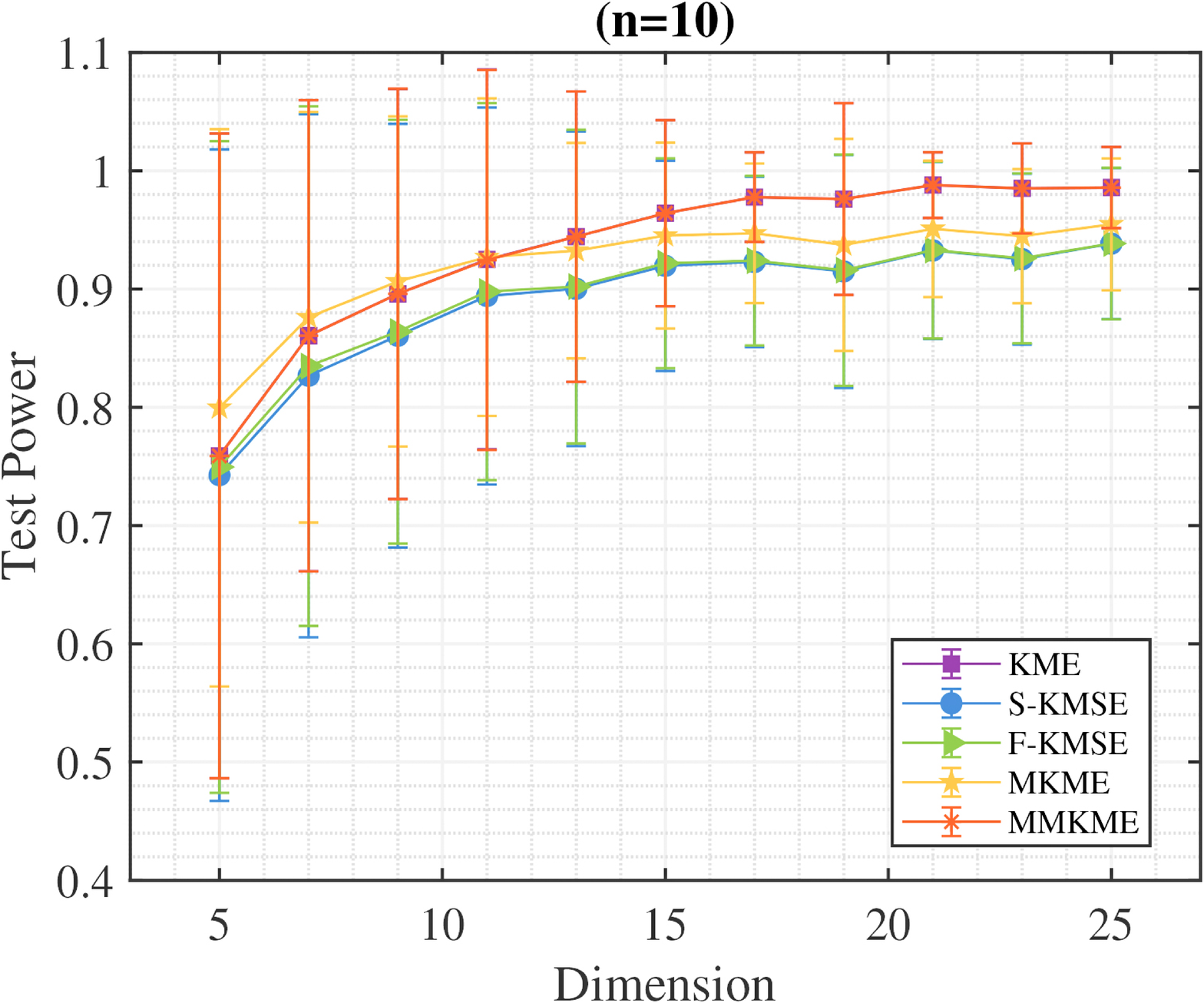}
\end{minipage}%
}%

\subfigure[]{
\begin{minipage}[t]{0.46\linewidth}
\centering
\includegraphics[width=6.3cm,height=5cm]{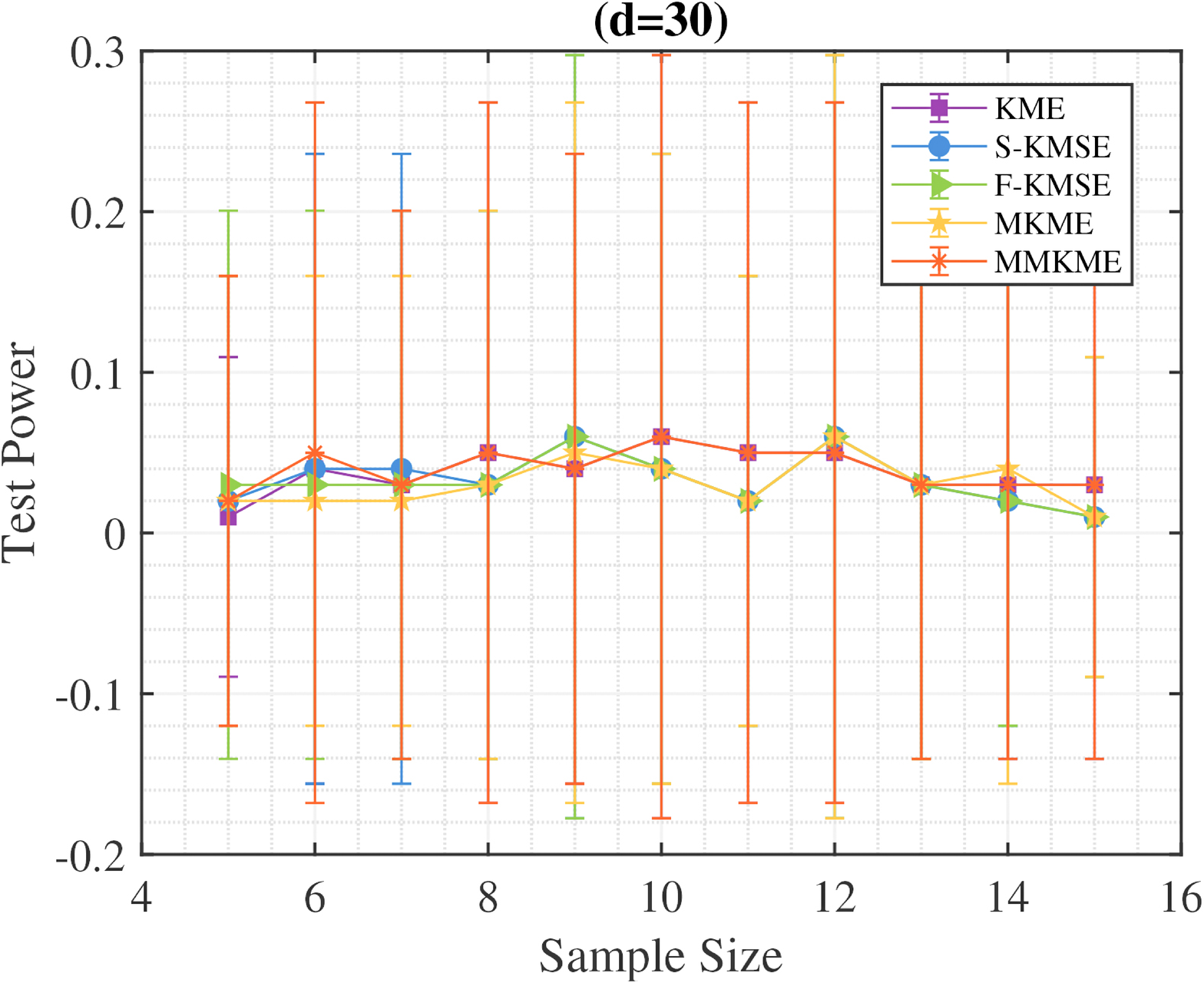}
\end{minipage}
}%
\subfigure[]{
\begin{minipage}[t]{0.46\linewidth}
\centering
\includegraphics[width=6.3cm,height=5cm]{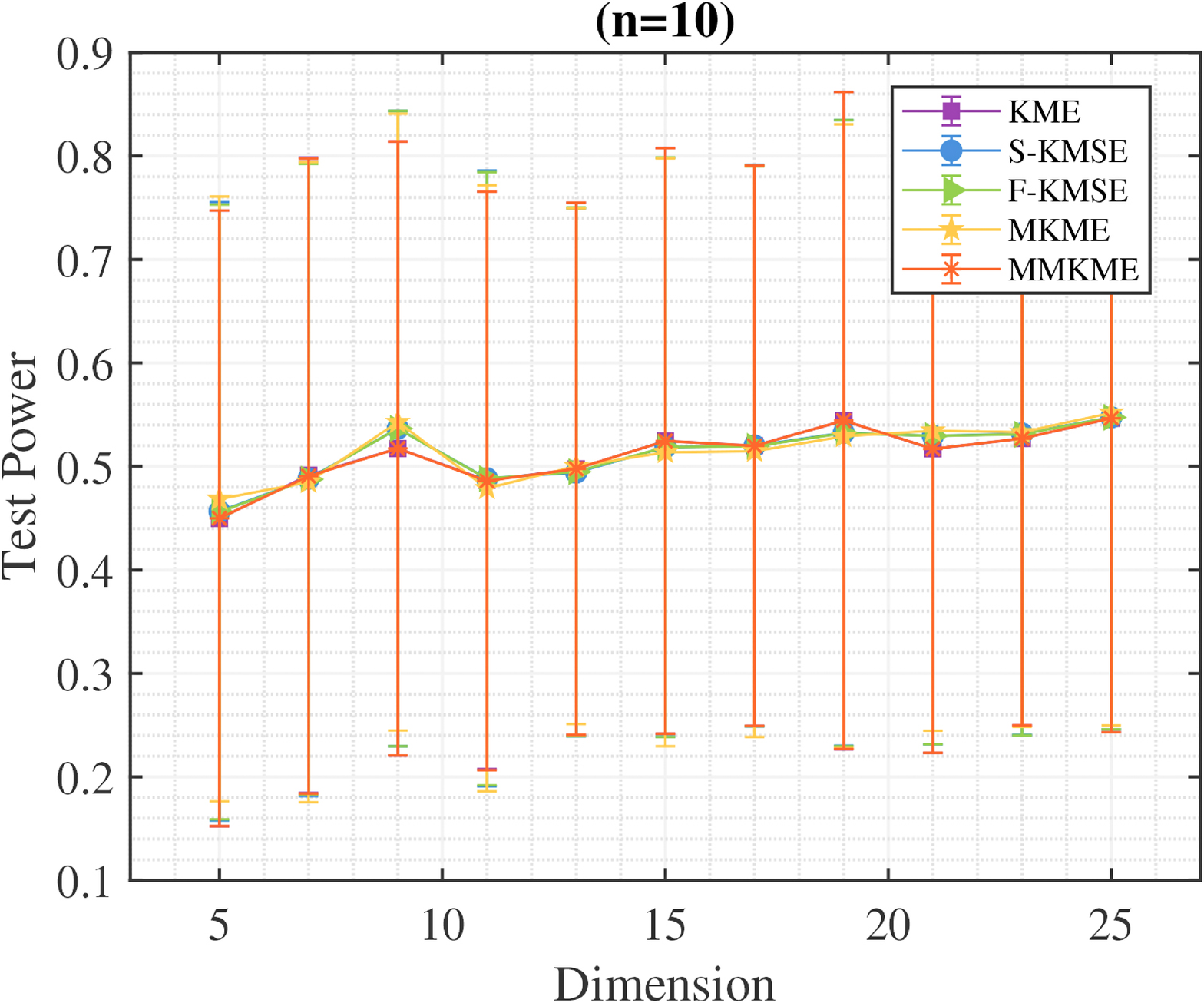}
\end{minipage}
}%
\centering
\caption{The performance about kernel two sample test achieved by KME, S-KMSE, F-KMSE, MKME, and MMKME over 500 trials. Specifically, the subfigure (a) and (b) correspond that two samples come from two different mixtures of Gaussian distributions. The subfigure (c) and (d) correspond that two samples come from two identical distributions.}
\label{fig:kernel_test}
\end{figure*}

\subsubsection{Kernel Two Sample Test}
The kernel two sample test is proposed to analyze and compare distributions, where we use to construct statistical tests to determine if two samples are drawn from different distributions \cite{gretton2007kernel,gretton2012kernel}. The basic idea of kernel two sample test is to first compute maximum mean discrepancy (MMD) \cite{gretton2007kernel} between two samples, and then determine whether two samples are drawn from different distributions according the distance. In this paper, we use a permutation test for simulation. For two samples $S'_1=\{x^1_1,x^1_2,\ldots,x^1_n\}$ and $S'_2=\{x^2_1,x^2_2,\ldots,x^2_n\}$, the procedure of the experiments is as follows: \textit{i}) compute the maximum mean discrepancy $m$ between $S'_1$ and $S'_2$; \textit{ii}) randomly permute the two samples and split them into two parts whose sample size is $n$; \textit{iii}) compute the new distance for permuted samples; \textit{iv}) repeat the step \textit{i} and step \textit{ii} for 1000 times to construct a null distribution. We employ significance test to show the test power. When the $p$-value is less than 0.05, we think the result is statistically significant.

In the experiment, we exploit different kinds of distributions: two different mixtures of Gaussian distributions and two identical distributions. The samples are generated with randomly generated means and different dimensions. The experimental results of two cases are presented in Figure~\ref{fig:kernel_test}. For the first case, as can be seen, the proposed marginalized kernel mean estimators achieve competitive performance compared with KME and always outperforms the shrinkage estimators S/F-KMSE. For the second case, all estimators achieve similar performance. Also, the proposed marginalized kernel mean estimators outperform in some circumstances. Thus, experimental results can justify our claims.
\begin{table*}[!t]
  \centering
  \normalsize
  \begin{tabular}{|c|ccccc|}
    \hline
    Dataset&KME&S-KMSE&F-KMSE&MKME&MMKME\\
    \hline
    wine&15.8832&15.6951&15.2647&15.2105&\textbf{15.0128}\\
    glass&9.3762&\textbf{9.3493}&9.8883&9.5733&\textbf{9.3488}\\
    bodyfat&27.4102&27.2059&27.9233&\textbf{26.9305}&26.9582\\
    svmguide2&27.3456&28.4625&27.2999&26.9327&\textbf{26.8865}\\
    ionosphere&33.2007&32.5539&32.5544 &32.5503&\textbf{32.4799}\\
    housing&15.7106&17.4248&11.2465&\textbf{10.6317}&12.3365\\
    sonar&77.6716&74.7398&76.0413&74.0239&\textbf{73.8970}\\
    specft&52.6763&49.6952&\textbf{49.6642} &49.7394 & 49.8805 \\
    bupa&6.7521&6.7382&6.7520&6.4515&\textbf{6.2730} \\
    lymphography&42.5804&46.8639&34.2432&27.1871&\textbf{26.0653}\\
    primary&35.8407&33.1766&34.6557&\textbf{26.9263}&28.6515 \\
    flag&37.7335&37.3129&14.1088&\textbf{13.5311}&14.8627 \\
    hayes&6.7454&5.9180&6.1711 &\textbf{5.6200}&5.8683\\
    wbdc&32.3002&27.1686&27.1380&27.1122&\textbf{26.6011}\\
    australian&5.5765&5.3334&5.2259&5.3703&\textbf{5.2407} \\
    \hline
  \end{tabular}
  \caption{Average negative log-likelihood on test points over 10 randomizations. The boldface represents the best performance.}
  \label{tab:kde}
\end{table*}

\begin{table*}[!t]
  \centering
  \normalsize
  \begin{tabular}{|c|c|ccccc|}
    \hline
    $\alpha$&$\eta$&KME&S-KMSE&F-KMSE&MKME&MMKME \\
    \hline
    \multirow{5}{*}{0.05}&0.10&0.070&0.060&0.075&0.060&\textbf{0.100} \\
   	&0.15&0.045&0.065&0.045&0.065&\textbf{0.240}   \\
   	&0.20&0.035&0.105&0.009&0.035&\textbf{0.485}   \\ 
   	&0.25&0.170&0.460&0.175&0.015&\textbf{0.715}   \\
   	&0.30&0.070&0.370&0.355&0.007&\textbf{0.880}   \\
    \hline
   	\multirow{5}{*}{0.10}&0.10&0.095&0.095&0.095&0.065&\textbf{0.175} \\
   	&0.15&0.095&0.120&0.125&0.095&\textbf{0.380}   \\
   	&0.20&0.075&0.220&0.245&0.075&\textbf{0.610}   \\ 
   	&0.25&0.070&0.460&0.460&0.070&\textbf{0.775}   \\
   	&0.30&0.160&0.665&0.635&0.160&\textbf{0.915}  \\
    \hline
   	\multirow{5}{*}{0.15}&0.10&0.125&0.150&0.140&0.100&\textbf{0.320} \\
   	&0.15&0.130&0.180&0.190&0.135&\textbf{0.455}   \\
   	&0.20&0.120&0.345&0.375&0.120&\textbf{0.695}   \\ 
   	&0.25&0.185&0.650&0.655&0.185&\textbf{0.835}   \\
   	&0.30&0.320&0.835&0.805&0.320&\textbf{0.940}  \\
    \hline
  \end{tabular}
   \caption{The power of HSIC achieved by different estimators. The experiments are conducted on the Eckerle dataset. The boldface represents the best performance.}
   \label{tab:hsic_e}
\end{table*}

\subsubsection{Kernel Density Estimation}
The density function can be estimated with kernel  mean estimators. Given a finite training sample, a better estimator can give rise to a more accurate density function. We perform density estimation via kernel mean matching \cite{song2008tailoring}. In this experiment, we employ fifteen UCI datasets. For these real datasets, we are neither aware of the underlying distribution nor the true kernel mean embedding. Thus, it seems to be hard to compare the performance of different estimators. Following \cite{song2008tailoring}, we address this issue by using the mixture of Gaussian to model the underlying distribution. Specifically, we first apply the $k$-means clustering algorithm \cite{krishna1999genetic} to constuct ten Gaussian distributions as our prototypes. We then minimize MMD between the estimator and prototypes. The mixture of Gaussian prototypes is the density function where we model the data. To compare the performance of different estimators, we report average NLL of the test data according the mixture of Gaussian. We use 30\% of the dataset as a test set. 

Note that, as the real-world datasets are complicated, it is hard to find a general bandwidth which is suitable for all circumstances. Therefore, we optimize the bandwidth of the RBF kernel. With an early-stopping strategy based on NLL, we search it in a range around median heuristic. This optimization will make the experiments more solid. The experimental results on UCI datasets are provided in Table \ref{tab:kde}. As can be seen, the proposed estimators, i.e., MKME and MMKME, outperform the baselines in thirteen out of fifteen 
cases. For the experiments on ``glass'' and ``specft'', MKME and MMKME also achieve competitive performance. 

\subsubsection{Hilbert-Schmit Independence Test}
Hilbert-Schmit independence test \cite{gretton2005measuring} is a test of whether significant statistical dependence is obtained by a kernel dependence measure, the Hilbert-Schmidt independence criterion (HSIC). More details about Hilbert-Schmit independence test can be found in \cite{gretton2007kernel}. In this paper, we follow the prior work \cite{ramdas2015nonparametric} and conduct experiments on the real-world dataset Eckerle \cite{eckerle1979circular}. The illustrations of the correlation of the dataset can be found in \cite{ramdas2015nonparametric}. We investigate the performance of estimators with different subsample percentages $\eta$. For all experiments, $\alpha\in\{0.05,0.10,0.15\}$ is chosen as the type-1 error (for choosing the threshold level of the null distribution’s right tail). For every setting of parameters of each experiment, power is calculated as the percentage of rejection over 200 repetitions (independent trials), with
2000 permutations per repetition (permutation testing to find the null distribution threshold at level $\alpha$). The detailed results are shown in Table \ref{tab:hsic_e}. Our marginalized estimator MMKME is significantly better than baselines in all settings. One possiable reason for MMKME always performs better than MKME is that the setting of multi-variables gives more freedom and thus the null distribution can be more precise. 

\section{Conclusion}\label{Sec:6}
Previous work has shown that there exists a large amount of estimators that are better than the standard kernel mean estimator. In this work, we propose a novel marginalized kernel mean estimator. Different from previous shrinkage estimators, marginalized kernel mean estimator is a first attempt to combine marginalization with kernel methods. The marginalized approach introduces implicit regularization in kernel mean estimation. Experimental results demonstrate that the proposed algorithm performs well in various tasks. In the future, we will explore the influence of using different marginalized corrupted distributions and different kernels.

\newpage
\bibliographystyle{plainnat}
\bibliography{bib}

\begin{thebibliography}{49}
\providecommand{\natexlab}[1]{#1}
\providecommand{\url}[1]{\texttt{#1}}
\expandafter\ifx\csname urlstyle\endcsname\relax
  \providecommand{\doi}[1]{doi: #1}\else
  \providecommand{\doi}{doi: \begingroup \urlstyle{rm}\Url}\fi

\bibitem[Andrew et~al.(2013)Andrew, Arora, Bilmes, and Livescu]{andrew2013deep}
Galen Andrew, Raman Arora, Jeff Bilmes, and Karen Livescu.
\newblock Deep canonical correlation analysis.
\newblock In \emph{ICML}, pages 1247--1255, 2013.

\bibitem[Andrieu and Moulines(2003)]{andrieu2003ergodicity}
C~Andrieu and E~Moulines.
\newblock Ergodicity of some adaptive markov chain monte carlo algorithm.
\newblock Technical report, Technical report, 2003.

\bibitem[Aronszajn(1950)]{aronszajn1950theory}
Nachman Aronszajn.
\newblock Theory of reproducing kernels.
\newblock \emph{Transactions of the American mathematical society}, 68\penalty0
  (3):\penalty0 337--404, 1950.

\bibitem[Bach and Jordan(2002)]{bach2002kernel}
Francis~R Bach and Michael~I Jordan.
\newblock Kernel independent component analysis.
\newblock \emph{Journal of Machine Learning Research}, 3\penalty0
  (Jul):\penalty0 1--48, 2002.

\bibitem[Berlinet and Thomas-Agnan(2011)]{berlinet2011reproducing}
Alain Berlinet and Christine Thomas-Agnan.
\newblock \emph{Reproducing kernel Hilbert spaces in probability and
  statistics}.
\newblock Springer Science \& Business Media, 2011.

\bibitem[Borgwardt et~al.(2006)Borgwardt, Gretton, Rasch, Kriegel,
  Sch{\"o}lkopf, and Smola]{borgwardt2006integrating}
Karsten~M Borgwardt, Arthur Gretton, Malte~J Rasch, Hans-Peter Kriegel,
  Bernhard Sch{\"o}lkopf, and Alex~J Smola.
\newblock Integrating structured biological data by kernel maximum mean
  discrepancy.
\newblock \emph{Bioinformatics}, 22\penalty0 (14):\penalty0 e49--e57, 2006.

\bibitem[Botev et~al.(2010)Botev, Grotowski, Kroese, et~al.]{botev2010kernel}
Zdravko~I Botev, Joseph~F Grotowski, Dirk~P Kroese, et~al.
\newblock Kernel density estimation via diffusion.
\newblock \emph{The Annals of Statistics}, 38\penalty0 (5):\penalty0
  2916--2957, 2010.

\bibitem[Chapelle et~al.(2001)Chapelle, Weston, Bottou, and
  Vapnik]{chapelle2001vicinal}
Olivier Chapelle, Jason Weston, L{\'e}on Bottou, and Vladimir Vapnik.
\newblock Vicinal risk minimization.
\newblock In \emph{NeurIPS}, pages 416--422, 2001.

\bibitem[Chen et~al.(2012)Chen, Xu, Weinberger, and Sha]{chen2012marginalized}
Minmin Chen, Zhixiang Xu, Kilian Weinberger, and Fei Sha.
\newblock Marginalized denoising autoencoders for domain adaptation.
\newblock \emph{arXiv preprint arXiv:1206.4683}, 2012.

\bibitem[Chen et~al.(2014)Chen, Weinberger, Sha, and
  Bengio]{chen2014marginalized}
Minmin Chen, Kilian Weinberger, Fei Sha, and Yoshua Bengio.
\newblock Marginalized denoising auto-encoders for nonlinear representations.
\newblock In \emph{ICML}, pages 1476--1484, 2014.

\bibitem[Dhillon et~al.(2004)Dhillon, Guan, and Kulis]{dhillon2004kernel}
Inderjit~S Dhillon, Yuqiang Guan, and Brian Kulis.
\newblock Kernel k-means: spectral clustering and normalized cuts.
\newblock In \emph{Proceedings of the tenth ACM SIGKDD international conference
  on Knowledge discovery and data mining}, pages 551--556, 2004.

\bibitem[Eckerle(1979)]{eckerle1979circular}
K~Eckerle.
\newblock Circular interference transmittance study.
\newblock \emph{National Institute of Standards and Technology (NIST), US
  Department of Commerce, USA}, 13, 1979.

\bibitem[Flaxman et~al.(2016)Flaxman, Sejdinovic, Cunningham, and
  Filippi]{flaxman2016bayesian}
Seth Flaxman, Dino Sejdinovic, John~P Cunningham, and Sarah Filippi.
\newblock Bayesian learning of kernel embeddings.
\newblock In \emph{UAI}, pages 182--191, 2016.

\bibitem[Fukumizu et~al.(2013)Fukumizu, Song, and Gretton]{fukumizu2013kernel}
Kenji Fukumizu, Le~Song, and Arthur Gretton.
\newblock Kernel bayes' rule: Bayesian inference with positive definite
  kernels.
\newblock \emph{Journal of Machine Learning Research}, 14\penalty0
  (1):\penalty0 3753--3783, 2013.

\bibitem[Gilks et~al.(1998)Gilks, Roberts, and Sahu]{gilks1998adaptive}
Walter~R Gilks, Gareth~O Roberts, and Sujit~K Sahu.
\newblock Adaptive markov chain monte carlo through regeneration.
\newblock \emph{Journal of the American Statistical Association}, 93\penalty0
  (443):\penalty0 1045--1054, 1998.

\bibitem[Gretton et~al.(2005)Gretton, Bousquet, Smola, and
  Sch{\"o}lkopf]{gretton2005measuring}
Arthur Gretton, Olivier Bousquet, Alexander Smola, and Bernhard Sch{\"o}lkopf.
\newblock Measuring statistical dependence with hilbert-schmidt norms.
\newblock In \emph{Algorithmic Learning Theory}, 2005.

\bibitem[Gretton et~al.(2007)Gretton, Fukumizu, Teo, Song, Sch{\"o}lkopf,
  Smola, et~al.]{gretton2007kernel}
Arthur Gretton, Kenji Fukumizu, Choon~Hui Teo, Le~Song, Bernhard Sch{\"o}lkopf,
  Alexander~J Smola, et~al.
\newblock A kernel statistical test of independence.
\newblock In \emph{NeurIPS}, pages 585--592, 2007.

\bibitem[Gretton et~al.(2012)Gretton, Borgwardt, Rasch, Sch{\"o}lkopf, and
  Smola]{gretton2012kernel}
Arthur Gretton, Karsten~M Borgwardt, Malte~J Rasch, Bernhard Sch{\"o}lkopf, and
  Alexander Smola.
\newblock A kernel two-sample test.
\newblock \emph{Journal of Machine Learning Research}, 13\penalty0
  (Mar):\penalty0 723--773, 2012.

\bibitem[James and Stein(1992)]{james1992estimation}
William James and Charles Stein.
\newblock Estimation with quadratic loss.
\newblock In \emph{Breakthroughs in statistics}, pages 443--460. Springer,
  1992.

\bibitem[Krishna and Murty(1999)]{krishna1999genetic}
K~Krishna and M~Narasimha Murty.
\newblock Genetic k-means algorithm.
\newblock \emph{IEEE Transactions on Systems, Man, and Cybernetics, Part B
  (Cybernetics)}, 29\penalty0 (3):\penalty0 433--439, 1999.

\bibitem[Li et~al.(2017)Li, Chang, Cheng, Yang, and P{\'o}czos]{li2017mmd}
Chun-Liang Li, Wei-Cheng Chang, Yu~Cheng, Yiming Yang, and Barnab{\'a}s
  P{\'o}czos.
\newblock Mmd gan: Towards deeper understanding of moment matching network.
\newblock In \emph{NeurIPS}, pages 2203--2213, 2017.

\bibitem[Lopez-Paz et~al.(2015)Lopez-Paz, Muandet, Sch{\"o}lkopf, and
  Tolstikhin]{lopez2015towards}
David Lopez-Paz, Krikamol Muandet, Bernhard Sch{\"o}lkopf, and Iliya
  Tolstikhin.
\newblock Towards a learning theory of cause-effect inference.
\newblock In \emph{ICML}, pages 1452--1461, 2015.

\bibitem[Maaten et~al.(2013)Maaten, Chen, Tyree, and
  Weinberger]{maaten2013learning}
Laurens Maaten, Minmin Chen, Stephen Tyree, and Kilian Weinberger.
\newblock Learning with marginalized corrupted features.
\newblock In \emph{ICML}, pages 410--418, 2013.

\bibitem[Miyato et~al.(2018)Miyato, Maeda, Koyama, and
  Ishii]{miyato2018virtual}
Takeru Miyato, Shin-ichi Maeda, Masanori Koyama, and Shin Ishii.
\newblock Virtual adversarial training: a regularization method for supervised
  and semi-supervised learning.
\newblock \emph{IEEE Transactions on Pattern Analysis and Machine
  Intelligence}, 41\penalty0 (8):\penalty0 1979--1993, 2018.

\bibitem[Muandet and Sch\"{o}lkopf(2013)]{Muandet:2013:OSM:3023638.3023684}
Krikamol Muandet and Bernhard Sch\"{o}lkopf.
\newblock One-class support measure machines for group anomaly detection.
\newblock In \emph{UAI}, pages 449--458, 2013.

\bibitem[Muandet et~al.(2012)Muandet, Fukumizu, Dinuzzo, and
  Sch{\"o}lkopf]{muandet2012learning}
Krikamol Muandet, Kenji Fukumizu, Francesco Dinuzzo, and Bernhard
  Sch{\"o}lkopf.
\newblock Learning from distributions via support measure machines.
\newblock In \emph{NeurIPS}, pages 10--18, 2012.

\bibitem[Muandet et~al.(2014{\natexlab{a}})Muandet, Fukumizu, Sriperumbudur,
  Gretton, and Sch{\"o}lkopf]{muandet2014kernel}
Krikamol Muandet, Kenji Fukumizu, Bharath Sriperumbudur, Arthur Gretton, and
  Bernhard Sch{\"o}lkopf.
\newblock Kernel mean estimation and stein effect.
\newblock In \emph{ICML}, pages 10--18, 2014{\natexlab{a}}.

\bibitem[Muandet et~al.(2014{\natexlab{b}})Muandet, Sriperumbudur, and
  Sch{\"o}lkopf]{muandet2014kernel_spectral}
Krikamol Muandet, Bharath Sriperumbudur, and Bernhard Sch{\"o}lkopf.
\newblock Kernel mean estimation via spectral filtering.
\newblock In \emph{NeurIPS}, pages 1--9, 2014{\natexlab{b}}.

\bibitem[Muandet et~al.(2016)Muandet, Sriperumbudur, Fukumizu, Gretton, and
  Sch{\"o}lkopf]{muandet2016kernel}
Krikamol Muandet, Bharath Sriperumbudur, Kenji Fukumizu, Arthur Gretton, and
  Bernhard Sch{\"o}lkopf.
\newblock Kernel mean shrinkage estimators.
\newblock \emph{Journal of Machine Learning Research}, 17\penalty0
  (1):\penalty0 1656--1696, 2016.

\bibitem[Muandet et~al.(2017)Muandet, Fukumizu, Sriperumbudur, Sch{\"o}lkopf,
  et~al.]{muandet2017kernel}
Krikamol Muandet, Kenji Fukumizu, Bharath Sriperumbudur, Bernhard
  Sch{\"o}lkopf, et~al.
\newblock Kernel mean embedding of distributions: A review and beyond.
\newblock \emph{Foundations and Trends{\textregistered} in Machine Learning},
  10\penalty0 (1-2):\penalty0 1--141, 2017.

\bibitem[Pan et~al.(2010)Pan, Tsang, Kwok, and Yang]{pan2010domain}
Sinno~Jialin Pan, Ivor~W Tsang, James~T Kwok, and Qiang Yang.
\newblock Domain adaptation via transfer component analysis.
\newblock \emph{IEEE Transactions on Neural Networks}, 22\penalty0
  (2):\penalty0 199--210, 2010.

\bibitem[Pearson(1900)]{Pearson1900On}
Karl Pearson.
\newblock On lines and planes of closest fit to points in space.
\newblock \emph{Philosophical Magazine}, 2\penalty0 (11):\penalty0 559--572,
  1900.

\bibitem[Ramdas and Wehbe(2015)]{ramdas2015nonparametric}
Aaditya Ramdas and Leila Wehbe.
\newblock Nonparametric independence testing for small sample sizes.
\newblock In \emph{IJCAI}, 2015.

\bibitem[Sch{\"o}lkopf et~al.(1997)Sch{\"o}lkopf, Smola, and
  M{\"u}ller]{scholkopf1997kernel}
Bernhard Sch{\"o}lkopf, Alexander Smola, and Klaus-Robert M{\"u}ller.
\newblock Kernel principal component analysis.
\newblock In \emph{International conference on artificial neural networks},
  pages 583--588. Springer, 1997.

\bibitem[Sch{\"o}lkopf et~al.(2001)Sch{\"o}lkopf, Herbrich, and
  Smola]{scholkopf2001generalized}
Bernhard Sch{\"o}lkopf, Ralf Herbrich, and Alex~J Smola.
\newblock A generalized representer theorem.
\newblock In \emph{ICOCLT}, pages 416--426. Springer, 2001.

\bibitem[Simard et~al.(1998)Simard, LeCun, Denker, and
  Victorri]{simard1998transformation}
Patrice~Y Simard, Yann~A LeCun, John~S Denker, and Bernard Victorri.
\newblock Transformation invariance in pattern recognition—tangent distance
  and tangent propagation.
\newblock In \emph{Neural networks: tricks of the trade}, pages 239--274.
  Springer, 1998.

\bibitem[Smola et~al.(2007)Smola, Gretton, Song, and
  Sch{\"o}lkopf]{smola2007hilbert}
Alex Smola, Arthur Gretton, Le~Song, and Bernhard Sch{\"o}lkopf.
\newblock A hilbert space embedding for distributions.
\newblock In \emph{International Conference on Algorithmic Learning Theory},
  pages 13--31. Springer, 2007.

\bibitem[Song et~al.(2007)Song, Smola, Gretton, Borgwardt, and
  Bedo]{song2007supervised}
Le~Song, Alex Smola, Arthur Gretton, Karsten~M Borgwardt, and Justin Bedo.
\newblock Supervised feature selection via dependence estimation.
\newblock In \emph{ICML}, pages 823--830, 2007.

\bibitem[Song et~al.(2008)Song, Zhang, Smola, Gretton, and
  Sch{\"o}lkopf]{song2008tailoring}
Le~Song, Xinhua Zhang, Alex Smola, Arthur Gretton, and Bernhard Sch{\"o}lkopf.
\newblock Tailoring density estimation via reproducing kernel moment matching.
\newblock In \emph{ICML}, pages 992--999, 2008.

\bibitem[Song et~al.(2009)Song, Huang, Smola, and Fukumizu]{song2009hilbert}
Le~Song, Jonathan Huang, Alex Smola, and Kenji Fukumizu.
\newblock Hilbert space embeddings of conditional distributions with
  applications to dynamical systems.
\newblock In \emph{ICML}, pages 961--968, 2009.

\bibitem[Song et~al.(2010)Song, Boots, Siddiqi, Gordon, and
  Smola]{Song:2010:HSE:3104322.3104448}
Le~Song, Byron Boots, Sajid~M. Siddiqi, Geoffrey Gordon, and Alex Smola.
\newblock Hilbert space embeddings of hidden markov models.
\newblock In \emph{ICML}, 2010.

\bibitem[Song et~al.(2011)Song, Gretton, Bickson, Low, and
  Guestrin]{song2011kernel}
Le~Song, Arthur Gretton, Danny Bickson, Yucheng Low, and Carlos Guestrin.
\newblock Kernel belief propagation.
\newblock \emph{arXiv preprint arXiv:1105.5592}, 2011.

\bibitem[Sriperumbudur et~al.(2008)Sriperumbudur, Gretton, Fukumizu, Lanckriet,
  and Sch{\"o}lkopf]{sriperumbudur2008injective}
Bharath~K Sriperumbudur, Arthur Gretton, Kenji Fukumizu, Gert Lanckriet, and
  Bernhard Sch{\"o}lkopf.
\newblock Injective hilbert space embeddings of probability measures.
\newblock In \emph{COLT}, pages 111--122. Omnipress, 2008.

\bibitem[Stein(1981)]{stein1981estimation}
Charles~M Stein.
\newblock Estimation of the mean of a multivariate normal distribution.
\newblock \emph{The annals of Statistics}, pages 1135--1151, 1981.

\bibitem[Szab{\'o} et~al.(2015)Szab{\'o}, Gretton, P{\'o}czos, and
  Sriperumbudur]{szabo2015two}
Zolt{\'a}n Szab{\'o}, Arthur Gretton, Barnab{\'a}s P{\'o}czos, and Bharath
  Sriperumbudur.
\newblock Two-stage sampled learning theory on distributions.
\newblock In \emph{Artificial Intelligence and Statistics}, pages 948--957,
  2015.

\bibitem[Tolstikhin et~al.(2017)Tolstikhin, Sriperumbudur, and
  Muandet]{tolstikhin2017minimax}
Ilya Tolstikhin, Bharath~K Sriperumbudur, and Krikamol Muandet.
\newblock Minimax estimation of kernel mean embeddings.
\newblock \emph{Journal of Machine Learning Research}, 18\penalty0
  (1):\penalty0 3002--3048, 2017.

\bibitem[Yang et~al.(2005)Yang, Frangi, Yang, Zhang, and Jin]{yang2005kpca}
Jian Yang, Alejandro~F Frangi, Jing-yu Yang, David Zhang, and Zhong Jin.
\newblock Kpca plus lda: a complete kernel fisher discriminant framework for
  feature extraction and recognition.
\newblock \emph{IEEE Transactions on Pattern Analysis and Machine
  Intelligence}, 27\penalty0 (2):\penalty0 230--244, 2005.

\bibitem[Yun et~al.(2019)Yun, Han, Oh, Chun, Choe, and Yoo]{yun2019cutmix}
Sangdoo Yun, Dongyoon Han, Seong~Joon Oh, Sanghyuk Chun, Junsuk Choe, and
  Youngjoon Yoo.
\newblock Cutmix: Regularization strategy to train strong classifiers with
  localizable features.
\newblock In \emph{ICCV}, pages 6023--6032, 2019.

\bibitem[Zhang et~al.(2018)Zhang, Cisse, Dauphin, and
  Lopez-Paz]{zhang2017mixup}
Hongyi Zhang, Moustapha Cisse, Yann~N Dauphin, and David Lopez-Paz.
\newblock mixup: Beyond empirical risk minimization.
\newblock In \emph{ICLR}, 2018.

\end{thebibliography}

\end{document}